%% file: main.tex
\documentclass{article}
\usepackage{iclr2027_conference,times}
\usepackage[utf8]{inputenc}
\usepackage[T1]{fontenc}
\usepackage{amsmath,amssymb,amsthm,mathtools}
\usepackage{thmtools,thm-restate}
\usepackage{algorithm,algpseudocode}
\usepackage{booktabs,graphicx,microtype,xcolor}
\usepackage{tikz}
\definecolor{citationblue}{RGB}{32,82,132}
\usepackage[colorlinks=true,citecolor=citationblue,linkcolor=citationblue,urlcolor=citationblue]{hyperref}
\usepackage[noabbrev]{cleveref}
\usepackage{enumitem}
\usepackage{placeins}
\input{macros}
\title{A Query Is Not a Commitment: Learning to Correct Expert Answers in Online Deferral}
\iclrfinalcopy
\author{\textbf{Yannis Montreuil}$^1$ \quad \textbf{Axel Carlier}$^2$ \\
\textbf{Lai Xing Ng}$^3$ \quad \textbf{Wei Tsang Ooi}$^1$ \\[4pt]
{\normalfont\small $^1$School of Computing, National University of Singapore} \\
{\normalfont\small $^2$IRIT, Toulouse INP, France} \\
{\normalfont\small $^3$Institute for Infocomm Research, A*STAR, Singapore} \\
{\normalfont\small \texttt{yannis.montreuil@u.nus.edu}}}
\hypersetup{
  pdftitle={A Query Is Not a Commitment: Learning to Correct Expert Answers in Online Deferral},
  pdfauthor={Yannis Montreuil, Axel Carlier, Lai Xing Ng, Wei Tsang Ooi}
}
\begin{document}
\raggedbottom
\maketitle
\lhead{Preprint}
\begin{abstract}
An inaccurate expert can still provide useful information after correction.
We study online learning to defer in which the learner chooses an expert
and fixes a correction function before purchasing its answer, then applies
that function to the answer received. The difficulty is that observed losses
reflect both expert quality and an unfinished correction: early errors can
discourage queries that would be valuable after learning.
We propose \ORUCB{}, which pools shared and expert-specific polynomial
responses. A bound on cumulative response-learning error calibrates
confidence-weighted risk regression and exploration, allowing the router
to account for this error when deciding which answers to buy.
Under bounded residuals and disagreements, a fixed feasible model of optimal responses,
and linear models of free and optimal queried risk, the calibrated algorithm
achieves high-probability pseudo-regret $O(\sqrt T\log(T+1))$ over $T$ rounds
for fixed problem parameters. The guarantee permits singular answer
distributions and misspecified shared responses; optimality is relative
to the bounded response class.
On four test streams, the selected cubic policy has lower
fee-inclusive cost than seven baselines that deploy answers unchanged.
Comparisons with a common correction learner examine routing, while
six-price comparisons measure cost and query rates.
\end{abstract}
\newlength{\templateparskip}
\newlength{\templateabovedisplayskip}
\newlength{\templatebelowdisplayskip}
\setlength{\templateparskip}{\parskip}
\setlength{\templateabovedisplayskip}{\abovedisplayskip}
\setlength{\templatebelowdisplayskip}{\belowdisplayskip}
\setlength{\parskip}{4.5pt}
\setlength{\abovedisplayskip}{6pt plus2pt minus3pt}
\setlength{\belowdisplayskip}{6pt plus2pt minus3pt}
\input{sections/01_introduction}
\input{sections/06_related_work}
\input{sections/02_background}

\input{sections/02_problem}
\input{sections/03_method}
\input{sections/04_theory}
\input{sections/05_experiments}
\input{sections/06_conclusion}
\setlength{\parskip}{\templateparskip}
\setlength{\abovedisplayskip}{\templateabovedisplayskip}
\setlength{\belowdisplayskip}{\templatebelowdisplayskip}
\clearpage
\input{sections/07_ai_use}
\bibliographystyle{iclr2027_conference}
\bibliography{references}
\clearpage
\appendix
\setlength{\jot}{5pt}
\input{appendix/a_protocol_and_dimensions}
\clearpage
\input{appendix/e_implementation}
\clearpage
\input{appendix/b_response_geometry}
\clearpage
\input{appendix/d_optimal_risk_proof}

\clearpage
\input{appendix/h_experiment_details}
\clearpage
\input{appendix/g_data_construction}
\clearpage
\input{appendix/i_related_details}
\end{document}

%% file: macros.tex
\newcommand{\E}{\mathbb{E}}

\newcommand{\R}{\mathbb{R}}
\newcommand{\cE}{\mathcal{E}}
\newcommand{\cF}{\mathcal{F}}

\newcommand{\cA}{\mathcal{A}}

\newcommand{\one}{\mathbf{1}}

\newcommand{\Reg}{\operatorname{Reg}}
\newcommand{\clip}{\operatorname{clip}}

\newcommand{\argmin}{\operatorname*{arg\,min}}

\newcommand{\diag}{\operatorname{diag}}

\newtheorem{theorem}{Theorem}[section]

\newtheorem{assumption}[theorem]{Assumption}

\crefname{assumption}{assumption}{assumptions}
\Crefname{assumption}{Assumption}{Assumptions}
\theoremstyle{remark}

\newcommand{\ORUCB}{\textnormal{\textsc{Response-Aware UCB}}}

%% file: sections/01_introduction.tex
\section{Introduction}
\label{sec:introduction}
Learning to defer (L2D) asks a router to choose between an internal
prediction and an external expert's answer
\citep{madras2018predict,mozannar2020consistent, montreuil2026beyond}. With several experts
and consultation fees, the router must decide whom to ask and whether the
expected reduction in prediction error is worth the price
\citep{mao2024regression}. Standard direct deferral deploys a purchased
answer unchanged. This ties the value of consultation to the accuracy of
the raw answer, even when its errors are predictable. Suppose a weather
model tracks day-to-day temperature changes but consistently forecasts
too high. Its uncorrected error and fee may make a query unattractive,
while removing the stable bias would make the same answer useful. Even a
router that knows the data law cannot realize this gain if it must deploy
the answer unchanged. The question is therefore not only \emph{whom to
ask}, but also \emph{how to use the answer} once it arrives.

We study this distinction through online L2D with \emph{post-answer
responses}. On round $t$, the learner sees its internal prediction $h_t$,
the public context, the available experts, and their fees. It either uses
$h_t$ or chooses an expert $j$ and fixes a correction function
$\Delta_{t,j}$. If it queries $j$, the final prediction is
\begin{equation}
\widehat Y_t=M_{t,j}+\Delta_{t,j}(M_{t,j}-h_t).
\label{eq:intro-response}
\end{equation}
The function is fixed before the query and evaluated after the answer
arrives. Only the selected expert's answer is observed. The outcome is
revealed after the final prediction on every round, providing feedback for
future corrections and query decisions. Thus the final prediction may use
the purchased answer, but the decision to buy that answer is made before
it is seen. We represent this sequential decision in Appendix~\ref{app:protocol}.

The router must decide whether to buy an expert's answer while it is still
learning how to correct that answer. In direct deferral, a purchase reveals
the loss of the queried action. With a learned response, it reveals the
loss under the learner's current correction, while the router
must assess the risk attainable with the best feasible response. The
conditional gap is nonnegative and can change with the response; a
confidence bound for centered observation noise alone does not cover it.
Early correction errors can therefore make a useful expert appear too
costly. If queries stop, its retained response loses direct training
observations. We propose \ORUCB{} to handle this moving target. It pools
bounded polynomial responses that can use the purchased answer nonlinearly:
a shared fit offers an initial proposal for a new expert, while an
identity-specific fit preserves what was learned when that expert returns.
The shared fit need not describe every expert correctly. We bound the
cumulative response-learning error relative to feasible identity responses,
then use that bound to calibrate confidence-weighted risk regression. This
limits the influence of observations acquired at high queried-risk
uncertainty; a confidence bonus favors queries when an expert's corrected
risk could still be low.

Our contributions are as follows.
(i) We define a conditional-law oracle that, like the learner, selects an
expert and fixes its response before seeing the answer. This comparator
measures the cost of learning both which answers to buy and how to use them.
(ii) Under fixed feasible models of optimal response coefficients and linear
models of free and optimal queried risk, calibrated \ORUCB{} has
high-probability pseudo-regret $O(\sqrt T\log(T+1))$ for fixed problem
parameters. The guarantee allows singular answer distributions and
misspecified shared responses. A robustness extension quantifies the cost
of approximate response and risk models when bounds on their approximation
errors are supplied for calibration.
(iii) On four test streams, the selected cubic policy
has lower cost than seven external baselines that deploy answers unchanged.
Comparisons with a common correction learner examine routing, while
six-price comparisons measure cost and query rates.

%% file: sections/06_related_work.tex
\section{Related work}
\label{sec}

\textbf{Learning to defer.}
L2D learns when to predict and when to rely on an expert
\citep{madras2018predict,mozannar2020consistent}.
In regression, routing compares the prediction losses and consultation
costs of several experts \citep{mao2024regression}. In classification,
\citet{montreuil2026beyond} separate classifier and expert scores to address
coupling in joint prediction--deferral training.
Online L2D extends routing to varying experts and selected-action feedback,
primarily in classification settings \citep{duy2026online}, while
time-series deferral also models changes in expert performance
\citep{montreuil2026slds}. Our setting differs in that querying an expert does
not end the decision process. After purchasing an answer, the learner must
also learn how to transform that answer into the final prediction. The
response function is fixed before acquisition, but its output depends on the
answer received. Errors in learning this response therefore affect both the
quality of the current prediction and which experts are queried in the
future.

\textbf{Online bandits.}
Contextual bandits learn action values from selected-action feedback using
linear confidence bounds \citep{li2010contextual,abbasi2011improved},
parameter sampling \citep{agrawal2013thompson}, or neural confidence models
\citep{zhou2020neuralucb}. Our comparisons cover these alternatives through
SleepingLinUCB, SharedLinUCB, LinTS, and NeuralUCB. SleepingLinUCB maintains a
separate model for each expert and restricts selection to available experts,
following the sleeping-bandit setting \citep{kleinberg2010sleeping}, while
SharedLinUCB uses a shared context--identity model. To account for changing
performance, we also compare with discounted regression
(D-LinUCB; \citealt{russac2019weighted}) and detector-triggered restarts
(CUSUM-LinUCB and GLR-LinUCB), adapted from the non-contextual methods of
\citet{liu2018change} and \citet{besson2022efficient}. All seven baselines use
our availability and feedback protocol.
Our confidence weighting follows \citet{he2022nearly}. The additional
analytical step is to bound the bias generated by learning the responses
and use this bound to calibrate acquisition.
Appendix~\ref{app:related-details} develops these connections.

%% file: sections/02_background.tex
\section{Background}
\label{sec:background}
\textbf{Learning to defer.}
In learning to defer (L2D), a policy decides whether to use an internal
prediction or delegate to an expert
\citep{madras2018predict,mozannar2020consistent}. The expert may be a human
or another model. For regression, let $Y\in\R$ be the outcome and
$M=(M_1,\ldots,M_K)\in\R^K$ the potential answers from $K\ge1$ expert identities.
The public input $x\in\mathcal X$ determines the internal prediction
$h(x)$, available identities $\mathcal E(x)\subseteq\{1,\ldots,K\}$ and
consultation fees $\beta_j(x)\ge0$.
An admissible policy, or \emph{router}, is a measurable map
$\pi:\mathcal X\to\{0,\ldots,K\}$ with
$\pi(x)\in\{0\}\cup\mathcal E(x)$. Action $0$ uses the internal prediction
$h(x)$; action $j$ deploys expert answer $M_j$ unchanged.
We call squared prediction error the \emph{loss} and loss plus the
consultation fee the \emph{cost}:
\begin{equation}
C^{\rm def}(x,Y,M;i)=
\begin{cases}
(Y-h(x))^2,&i=0,\\
(Y-M_j)^2+\beta_j(x),&i=j\in\mathcal E(x).
\end{cases}
\label{eq:background-deferral-cost}
\end{equation}
L2D minimizes $\E[C^{\rm def}(X,Y,M;\pi(X))]$ over a fixed
nonempty class $\Pi$ of admissible routers, where $X$ is the random public
input \citep{mao2024regression, mao2024principled}. Querying an expert improves expected cost over
the internal prediction when its conditional loss reduction exceeds its
fee.

\textbf{Online decisions and feedback.}
Online L2D learns from sequential feedback \citep{duy2026online}.
On round $t$ of $T\ge1$ rounds, the pre-query information $\cF_t$ contains
the completed history and current public input $x_t$; mathematically, it
is a sigma-field. A history-dependent router chooses
$I_t=\pi_t(x_t)\in\{0\}\cup\mathcal E(x_t)$: $I_t=0$ uses $h(x_t)$,
while $I_t=j$ selects expert $j$. These choices are called \emph{arms}
in bandit terminology. Availability may change while expert identities
persist. The internal predictor may also learn, with its current prediction
included in $x_t$.
Write $M_t=(M_{t,1},\ldots,M_{t,K})$ for the potential-answer vector.
Public quantities such as $x_t$ may be random through the history; they
are known when conditioning on $\cF_t$.

Feedback determines which losses can be used to update the router.
Full information reveals every available action's loss; bandit feedback
reveals only the chosen action's loss \citep{auer2002finite}.
If outcomes are revealed on every round but expert answers only upon
purchase, the internal loss is observed even when an expert is selected.
The losses of unqueried experts are generally unavailable.

For $\Pi$ fixed before interaction, realized direct-deferral regret
compares cumulative cost with the best router in hindsight
\citep{cesabianchi2006prediction}:
\begin{equation}
\Reg_T^{\rm def}(\Pi)=
\sum_{t=1}^T C^{\rm def}(x_t,Y_t,M_t;I_t)
-\inf_{\pi\in\Pi}\sum_{t=1}^T
C^{\rm def}(x_t,Y_t,M_t;\pi(x_t)).
\label{eq:background-direct-regret}
\end{equation}
The infimum is outside the sum: one router is used throughout, but it may
select different experts at different inputs. Both costs use the same
realized sequence, including unobserved potential answers.
A \emph{conditional risk} is expected prediction loss given $\cF_t$.
Our \emph{conditional pseudo-regret} in \cref{sec:problem} compares
expected costs with an oracle's best feasible action and response on each
round, chosen before any current answer is revealed.

%% file: sections/02_problem.tex
\section{Protocol and comparator}
\label{sec:problem}
We now specify the allowed corrections and the oracle against which we
measure the cost of learning to query and use an answer.

\textbf{Protocol.}
Write $h_t=h(x_t)$, $\cE_t=\mathcal E(x_t)$ and
$\beta_{t,j}=\beta_j(x_t)$. Before querying, $Y_t$ and the potential
answers $(M_{t,j})_{j\in\cE_t}$ are hidden. We assume non-intervention:
purchasing an answer does not change their conditional joint law, as when
a forecast does not affect the weather. Using only $\cF_t$, the learner
fixes $I_t$ and a response function. Action $0$ deploys $h_t$; a query to
$j$ pays $\beta_{t,j}$, reveals only $M_{t,j}$ and applies the correction
in \eqref{eq:intro-response}. The outcome is then revealed on every round.
Our policy is deterministic, with public orders for ties and registration.

\textbf{Disagreement.}
Assume $|Y_t-h_t|\le R$ and $|M_{t,j}-h_t|\le D$ almost surely for every
round and active expert, with known $R,D>0$.
Set $G=R+D$, $d_{t,j}=M_{t,j}-h_t$ and
$e_{t,j}=Y_t-M_{t,j}$; then $|e_{t,j}|\le G$.
We learn the response by predicting the expert residual $e_{t,j}$ from
disagreement $d_{t,j}$.
Since $M_{t,j}=h_t+d_{t,j}$ and $h_t$ is public, centering loses no
information about the answer given $\cF_t$. It gives a bounded response
input without requiring bounded absolute prediction levels.

\textbf{Response class.}
We model the correction as a low-degree polynomial in disagreement.
Polynomial features allow nonlinear adjustments while keeping the model
linear in its coefficients \citep{hu2020smooth,liu2021smooth}, so fitting
the response by squared loss remains a convex problem. Keeping the degree
low also limits how many coefficients must be learned from each expert's
purchased answers. Because purchases depend on earlier fits, we use an
online least-squares comparison that does not require independent samples
\citep{hazan2023online}. It bounds cumulative fitting loss relative to a
fixed feasible model, allowing us to account for response-learning error
in the routing analysis (\cref{lem:constrained-regret}).

To keep the corrections bounded, we use scaled Legendre features and
restrict the coefficient norm. For an integer order $q\ge1$, let
$\mathbb B_q=\{a\in\R^{q+1}:\|a\|_2\le1\}$, with $\|\cdot\|_2$ the
Euclidean norm. Using the degree-$k$ Legendre polynomial $P_k$, define
\begin{equation}
\phi_q(d)=\frac{(1,DP_1(d/D),\ldots,DP_q(d/D))^\top}{\sqrt{1+qD^2}},
\qquad \Delta_a(d)=Ga^\top\phi_q(d).
\label{eq:polynomial-dictionary}
\end{equation}
For a query to expert $j$, the learner fixes $a_{t,j}\in\mathbb B_q$
before purchase. Its deployed prediction is
\begin{equation}
\widehat Y_t=
\begin{cases}
h_t, & \text{if } I_t=0,\\
M_{t,j}+\Delta_{a_{t,j}}(d_{t,j}), & \text{if } I_t=j\in\cE_t.
\end{cases}
\label{eq:deployed-prediction}
\end{equation}
For $d\in[-D,D]$, the bound $|P_k(d/D)|\le1$
\citep{szego1975orthogonal} gives $\|\phi_q(d)\|_2\le1$ and
$|\Delta_a(d)|\le G$ for $a\in\mathbb B_q$. The factor $D$ preserves
the linear coordinate $DP_1(d/D)=d$. Since $|e_{t,j}|\le G$, the queried prediction loss
$(e_{t,j}-\Delta_a(d_{t,j}))^2$ is at most $4G^2$. This bound supplies
bounded labels for risk estimation. The scaling and coefficient ball
define the response budget. The normalization depends on $q$, so the
bounded classes need not be nested.
The analysis requires bounded features and linear coefficients, not
orthogonality under the answer distribution, and extends to other fixed
bounded dictionaries under their corresponding response and risk models.

For the experiments, we choose $q=3$: four coefficients at a fixed context
capture bias correction, linear rescaling, and quadratic and cubic
adjustments. The cubic term allows the rescaling to vary with disagreement
magnitude.

\textbf{Oracle and regret.}
The comparator class consists of all $\cF_t$-measurable rules choosing
free prediction or a pair $(j,a)$ with $j\in\cE_t$ and $a\in\mathbb B_q$.
Routing may use the full public history, but the correction must belong
to the specified response class. Write
$L_t^0=\E[(Y_t-h_t)^2\mid\cF_t]$ for the free risk. The best queried
prediction risk for expert $j$ is
\begin{equation}
W_{t,j}^*=\min_{a\in\mathbb B_q}
\E[(e_{t,j}-Ga^\top\phi_q(d_{t,j}))^2\mid\cF_t].
\label{eq:risks}
\end{equation}
The minimization fixes one coefficient vector using $\cF_t$; its loss is
averaged over the still-hidden outcome and answer. Compactness and continuity
give a minimizer; Appendix~\ref{app:protocol} specifies measurable versions.
Since $L_t^0\le R^2$, a query fee of at least $R^2$ cannot improve on free
prediction. Restricting queries to
$\cA_t=\{j\in\cE_t:\beta_{t,j}<R^2\}$ therefore preserves the oracle cost
$C_t^*=\min\{L_t^0,\min_{j\in\cA_t}(W_{t,j}^*+\beta_{t,j})\}$,
with an empty inner minimum equal to $+\infty$.
The oracle knows the conditional law and may change expert and response
with context, but must choose both before any current answer is revealed. Set $\beta_{t,0}=0$ and
$C_t=(Y_t-\widehat Y_t)^2+\beta_{t,I_t}$. Our pseudo-regret is
\begin{equation}
\Reg_T=\sum_{t=1}^T\{\E[C_t\mid\cF_t]-C_t^*\}.
\label{eq:regret}
\end{equation}
Unlike hindsight regret over the fixed direct-deferral class $\Pi$ in
\cref{sec:background}, this compares conditional expected costs with the
best feasible action and response on each round. The oracle chooses before
feedback, following the conditional-mean convention of linear contextual
bandits \citep{abbasi2011improved,he2022nearly}.

%% file: sections/03_method.tex
\section{\ORUCB}
\label{sec:method}
\ORUCB{} uses past purchases to learn both how to correct an expert's
answer and whether that answer is worth buying. The difficulty is that
an expert's early losses reflect an unfinished correction, not just the
risk attainable with its best response. We first fit the corrections,
then bound the error incurred while learning them, and use that bound
to construct the query rule.

\subsection{Learning the response}
Each expert has an identity-specific fit that retains what was learned
from its purchased answers. A second, shared fit uses purchases from all
identities and can help with a new expert when their responses have common
structure. We combine these two proposals according to their past losses
on the selected identity, without requiring the shared model to be correct.

\textbf{Response fits.}
Let $\ell_{t,j}\in\R^{d_\ell}$ and $s_{t,j}\in\R^{d_s}$ be public features
of norm at most one for the identity and shared fits, respectively.
They determine the correction coefficients before purchase; $\phi_q(d_{t,j})$
determines how those coefficients act on the answer after purchase.
For example, an intercept and a bounded temperature covariate let a
coefficient vary with the public temperature. With regularizers
$\lambda_I,\lambda_P>0$, we fit the normalized residuals by
\begin{align}
\Theta_{t,j}
&=\argmin_{\|\Theta\|_F\le1}
\left\{\lambda_I\|\Theta\|_F^2+
\sum_{r<t:I_r=j}
\left(\frac{e_{r,j}}G-\ell_{r,j}^\top\Theta\phi_q(d_{r,j})\right)^2\right\},
\label{eq:identity-response}\\
\Psi_t
&=\argmin_{\Psi\in\R^{d_s\times(q+1)}}
\left\{\lambda_P\|\Psi\|_F^2+
\sum_{r<t:I_r\ne0}
\left(\frac{e_{r,I_r}}G-s_{r,I_r}^\top\Psi\phi_q(d_{r,I_r})\right)^2\right\}.
\label{eq:population-response}
\end{align}
Here $\Theta\in\R^{d_\ell\times(q+1)}$ and $\|\Theta\|_F$ is the
Euclidean norm of its entries. The regularizers make both minimizers unique.
The identity proposal $a^H_{t,j}=\Theta_{t,j}^\top\ell_{t,j}$ is feasible
because $\|a^H_{t,j}\|_2\le\|\Theta_{t,j}\|_F\|\ell_{t,j}\|_2\le1$.
We project the shared proposal onto the same coefficient ball:
\begin{equation}
a^P_{t,j}=\frac{\Psi_t^\top s_{t,j}}
{\max\{1,\|\Psi_t^\top s_{t,j}\|_2\}}.
\label{eq:population-projection}
\end{equation}
Superscripts $H$ and $P$ denote identity history and population.
The identity fit is constrained during optimization so that its cumulative
loss can be compared with any fixed feasible identity model
(\cref{lem:constrained-regret} in Appendix~\ref{app:optimal-risk-proof}).
The shared fit only needs to supply a feasible alternative, so projection
suffices.

\textbf{Pooling.}
To adapt to whether an expert benefits from shared information, we deploy
a convex mixture of the two coefficient proposals:
\begin{equation}
a_{t,j}=w^H_{t,j}a^H_{t,j}+w^P_{t,j}a^P_{t,j},
\qquad w^H_{t,j},w^P_{t,j}\ge0,\quad w^H_{t,j}+w^P_{t,j}=1.
\label{eq:pool}
\end{equation}
We initialize both weights at $1/2$ for each new identity. After querying $j$
and observing the outcome, we evaluate the pre-query proposals using
$\mathcal L_t^u=(e_{t,j}/G-(a^u_{t,j})^\top\phi_q(d_{t,j}))^2/4$,
$u\in\{H,P\}$. We update the weights by exponential weighting
\citep{cesabianchi2006prediction}, multiplying each by
$\exp(-\eta\mathcal L_t^u)$ and renormalizing, with $0<\eta\le1/2$.
Other identities keep their weights. Both losses use the same purchased
answer, so comparing proposals needs no additional query.

The mixture stays in $\mathbb B_q$ by convexity. For the stated learning-rate range,
its cumulative squared loss over queries to each identity is at most that of either
proposal sequence plus $4G^2\log2/\eta$.
\Cref{lem:pool} in Appendix~\ref{app:optimal-risk-proof} proves this
comparison using exp-concavity \citep{hazan2023online}.
Thus the shared proposal can help without invalidating the guarantee
provided by the identity fit. All fits use only past observations, and
identity histories and pool weights persist during absence.

\subsection{Accounting for response error}
An expert may be worth buying after correction even when its early
corrected predictions are poor. To account for this distinction, we separate
the loss we observe from the optimal risk we want to estimate.
On a queried round, the normalized deployed loss is
\begin{equation}
Z_t=(Y_t-\widehat Y_t)^2/(4G^2)\in[0,1].
\label{eq:queried-label}
\end{equation}
We define the response-learning gap by
$\rho_t=(\E[(Y_t-\widehat Y_t)^2\mid\cF_t]-W_{t,I_t}^*)/G^2$
when $I_t\ne0$, and set $\rho_t=0$ otherwise. Since the deployed response
belongs to the oracle's coefficient ball, queried rounds satisfy
\begin{equation}
\E[Z_t\mid\cF_t]=W_{t,I_t}^*/(4G^2)+\rho_t/4,\qquad 0\le\rho_t\le4.
\label{eq:upward-bias}
\end{equation}
The first term is the router's target; the second is the error from learning
the response. We cannot observe $\rho_t$ because the optimal risk is unknown.
Instead of estimating each gap, we bound their cumulative size. This
requires a fixed feasible identity model to represent optimal responses
across contexts, as formalized in \cref{ass:response-model}.

\begin{restatable}[Cumulative response error]{lemma}{responsecertificate}
\label{lem:response-certificate}
Under the bounded protocol of \cref{sec:problem} and
\cref{ass:response-model}, use the exact response fits and pooling updates
of \ORUCB{} with $0<\eta\le1/2$. For every $\delta\in(0,1)$, with probability
at least $1-\delta/3$ over the interaction, simultaneously for all $t$,
\begin{equation}
\sum_{s\le t}\rho_s\le E_t,
\label{eq:response-certificate-main}
\end{equation}
where $E_t$ is the query-count bound in \cref{eq:identity-envelope}.
For $t\le T$, it satisfies $E_t\le\overline E_T$, with the explicit
horizon bound $\overline E_T$ given in \cref{eq:weighted-envelope}.
\end{restatable}
Appendix~\ref{app:response-certificate-proof} proves the lemma by combining
the constrained-fitting and pooling comparisons, then converting realized
loss differences to conditional risk differences.
On this event, $G^2E_t$ bounds the cumulative conditional excess prediction
risk caused by learning the responses. It grows only logarithmically with the
number of purchases when the dimensions, registry capacity, regularizers,
pool rate and $\delta$ are fixed. The calibrated router uses the deterministic
horizon bound $\overline E_T$ defined in \cref{sec:theory}, not the
unknown individual gaps.

\subsection{Choosing an action}
We now estimate whether an expert's optimal queried risk plus its fee
can beat free prediction. The response-error bound lets us account for
the bias in purchased losses when making this comparison.

\textbf{Risk estimates.}
We use public features
$x_t^0\in\R^{d_0}$ and $c_{t,j}\in\R^{d_c}$ of norm at most one.
Unlike $\ell_{t,j},s_{t,j}$, which determine correction coefficients,
these features predict free and optimal queried risk, respectively.
All four feature vectors are known under
$\cF_t$ and cannot use the current answer; they may use the same public input.
The guarantee assumes that both normalized risks are linear in these
features (\cref{ass:risk-model}); this is separate from the response-model
assumption above. The free loss is observed every round, whereas $Z_t$ is
available only after a query.
With positive-definite priors $\Lambda_0,\Lambda_*$ and
$c_s=c_{s,I_s}$ on queried rounds, we use the full-history ridge estimates
below, where $\omega_s$ is a sample weight defined next:
\begin{align}
V_t^0&=\Lambda_0+\sum_{s<t}x_s^0(x_s^0)^\top,&
\widehat\theta_t^0&=(V_t^0)^{-1}\sum_{s<t}x_s^0\frac{(Y_s-h_s)^2}{R^2},
\nonumber\\
V_t^*&=\Lambda_*+\sum_{s<t:I_s\ne0}\omega_sc_sc_s^\top,&
\widehat\theta_t^*&=(V_t^*)^{-1}\sum_{s<t:I_s\ne0}\omega_sc_sZ_s.
\label{eq:risk-estimators}
\end{align}
The design matrices $V_t^0,V_t^*$ record the observed feature directions;
the priors regularize directions with little data.

\textbf{Confidence weighting.}
To turn the cumulative error bound into a useful confidence bound, we also
control each observation's influence. A biased observation in a poorly
measured direction can strongly affect later risk estimates.
As in linear UCB
\citep{abbasi2011improved}, we measure uncertainty in the queried-risk
direction by
$u_{t,j}=\sqrt{c_{t,j}^\top(V_t^*)^{-1}c_{t,j}}$.
For a threshold $\tau>0$ fixed before interaction, we assign each queried
round a weight before observing the selected answer:
\begin{equation}
\omega_t=\min\{1,\tau/u_{t,I_t}\},
\label{eq:query-weight}
\end{equation}
using one at zero width. Whereas $w^H_{t,j},w^P_{t,j}$ mix response
coefficients, $\omega_t$ weights one queried-risk observation.
This confidence-weighting rule comes from corrupted linear bandits
\citep{he2022nearly}. It caps the product $\omega_tu_{t,I_t}$ at $\tau$,
so the total response error in \cref{lem:response-certificate} controls its
influence on the risk estimate. \Cref{eq:main-bias-contraction} gives the
bound; Appendix~\ref{app:risk-confidence} proves the resulting confidence
guarantee. Both the design and target sum use the same sample weight.
Response fits and proposal-loss updates remain unweighted: they learn the
correction directly, rather than treating its current loss as optimal risk.

\textbf{Acquisition.}
We express the forecasts in squared-loss units so they can be compared
with fees:
$\widehat L_t^0=R^2\clip_{[0,1]}((x_t^0)^\top\widehat\theta_t^0)$ and
$\widetilde W_{t,j}=4G^2c_{t,j}^\top\widehat\theta_t^*$.
Clipping truncates to the indicated interval and cannot increase the free
forecast's error because $L_t^0\in[0,R^2]$. The queried forecast stays
unclipped, as used in its confidence bound.

For a confidence multiplier $\alpha\ge0$, we use the acquisition scores
\begin{equation}
S_{t,j}=\widehat L_t^0-\widetilde W_{t,j}-\beta_{t,j}
+4\alpha G^2u_{t,j},\qquad S_{t,0}=0,\qquad j\in\cA_t.
\label{eq:score}
\end{equation}
The first three terms estimate the saving from consultation, including its
fee. The bonus favors queries whose optimal risk could be lower than the
current forecast. Under the assumptions and calibration in
\cref{thm:main-regret}, $\widetilde W_{t,j}-4\alpha G^2u_{t,j}$ is a
lower confidence bound for $W_{t,j}^*$, accounting for both label noise
and response-learning bias. No exploration bonus is needed for free
prediction because its loss is observed under every action; its estimation
error is still included in the regret bound.

We choose the largest score, preferring free prediction on an exact tie and
then the earliest registered expert. Thus we query only when the selected
expert's score is positive. We deploy its frozen pooled response
as in \cref{eq:deployed-prediction}, or use $h_t$ if no query is selected.
After observing $Y_t$, we update the free-risk fit on every round and the
queried-risk fit, response fits and selected identity's pool weights only
after a purchase.

The theorem couples $\alpha$ and $\tau$ through
\cref{eq:main-calibration}; the benchmarks retain the same updates but
select these constants independently by cost on the validation block.
Appendix~\ref{app:implementation} gives the benchmark
feature construction, Algorithm~\ref{alg:optimal-risk}, and the recursive
implementation.

%% file: sections/04_theory.tex
\section{Theory}
\label{sec:theory}
Can we learn which answers are worth purchasing while still learning how
to correct them? We answer this question by bounding regret against the
conditional response oracle in \cref{eq:regret}. The guarantee includes
both query mistakes and the prediction error incurred while learning the
responses, rather than assuming that the corrections are already known.

\subsection{Model assumptions}
We use one model to learn optimal corrections across contexts and another
to estimate their risk before purchasing an answer. Both assumptions hold
almost surely for every round and eligible expert; public contexts and
availability may depend on history.

\begin{assumption}[A feasible model of optimal responses]
\label{ass:response-model}
For every identity $j$, there is a fixed
$\Theta_j^*\in\R^{d_\ell\times(q+1)}$ with $\|\Theta_j^*\|_F\le1$ such that
$(\Theta_j^*)^\top\ell_{t,j}$ minimizes the conditional queried risk over the
entire ball $\mathbb B_q$ in \cref{eq:risks}.
\end{assumption}
We model the \emph{optimal coefficients} across contexts with one fixed
matrix per identity. This connects the fixed-comparator guarantee of
online regression \citep{hazan2023online} to the conditional response
oracle. The conditional mean residual itself need not be polynomial:
the model represents a best correction within the prescribed ball.
The guarantee also permits a misspecified shared model and singular
answer covariance. Appendix~\ref{app:assumption-context} relates this
condition to realizability in contextual bandits.

\begin{assumption}[Linear free and optimal queried risks]
\label{ass:risk-model}
Fixed parameters $\theta_0^*\in\R^{d_0}$ and $\theta_*\in\R^{d_c}$ satisfy
\begin{equation}
L_t^0/R^2=(x_t^0)^\top\theta_0^*,\qquad
W_{t,j}^*/(4G^2)=c_{t,j}^\top\theta_*.
\label{eq:linear-risk-model}
\end{equation}
Known finite constants $Q_0,Q_*$ bound
$\|\theta_0^*\|_{\Lambda_0}\le Q_0$ and $\|\theta_*\|_{\Lambda_*}\le Q_*$,
where $\|v\|_\Lambda=\sqrt{v^\top\Lambda v}$.
\end{assumption}
Fixed linear conditional means, bounded features and parameter-norm bounds
are used in OFUL \citep{abbasi2011improved}, linear Thompson sampling
\citep{agrawal2013thompson} and confidence-weighted linear bandits
\citep{he2022nearly}. We apply this structure to free risk and risk
\emph{after optimizing the response}; the latter is modeled separately
from the optimal coefficients. Both models allow context dependence
through their public features. For example, take $h=0$, $Y$ uniform on
$\{-0.3,0.3\}$, and an expert answering $M=Y+0.2+0.1x$ for public
$x\in[-1,1]$. With $q=3$, $R=0.3$, $D=0.7$, and
$\ell=(1,x)^\top/\sqrt2$, its optimal correction $-0.2-0.1x$ has one fixed
coefficient matrix with squared Frobenius norm $0.247<1$. The free and
optimal queried risks are $0.09$ and zero. Appendix~\ref{app:contextual-example}
gives the matrix and risk features satisfying both assumptions.

\subsection{Confidence calibration}
We choose the exploration multiplier $\alpha$ and weight threshold $\tau$
to cover two sources of estimation error: centered observation noise and
the upward bias from unfinished responses in \cref{eq:upward-bias}.
For a horizon $T\ge1$ and failure probability $\delta\in(0,1)$,
\cref{lem:response-certificate} gives
$\sum_{t\le T}\rho_t\le\overline E_T$ with probability at least
$1-\delta/3$, using the following bound known before interaction:
\begin{equation}
\overline E_T=2\left[K\lambda_I+8Kp_I
\log\left(1+\frac{T}{Kp_I\lambda_I}\right)
+\frac{4K\log2}{\eta}\right]+64\log\frac3\delta.
\label{eq:weighted-envelope}
\end{equation}
Here $p_I=d_\ell(q+1)$ counts coefficients per identity. The terms account
for fitting the identity responses, pooling the two proposals, and
converting observed loss comparisons to conditional risk comparisons
(Appendix~\ref{app:response-budget}).

For centered noise and regularization, self-normalized regression bounds
\citep{abbasi2011improved} give confidence radii $B_0$ for free risk and
$B_*$ for queried risk.
For $v\in\{0,*\}$, set $\lambda_v=\lambda_{\min}(\Lambda_v)>0$ and
$d_*=d_c$, and take
\begin{equation}
B_v=\tfrac12\sqrt{d_v\log\left(1+\frac{T}{d_v\lambda_v}\right)
                    +2\log(3/\delta)}+Q_v.
\label{eq:main-noise-radius}
\end{equation}
The square-root term covers noise; $Q_v$ covers regularization bias.
For queried risk, the weights cap $\omega_tu_{t,I_t}$ at $\tau$
\citep{he2022nearly}. The response-learning contribution to the radius is
therefore at most $\tau\overline E_T/4$
(\cref{eq:main-bias-contraction}). We balance this bound with $B_*$:
\begin{equation}
\tau=\frac{4B_*}{\overline E_T},\qquad
\alpha=B_*+\frac{\tau\overline E_T}{4}=2B_*.
\label{eq:main-calibration}
\end{equation}
On the joint confidence event,
$\widetilde W_{t,j}-4\alpha G^2u_{t,j}$ is then a lower confidence bound
for optimal queried risk, accounting for both sources of error.
Appendix~\ref{app:risk-confidence} proves the simultaneous confidence
statement.

\subsection{Regret guarantee}
\begin{restatable}[Regret of response-aware routing]{theorem}{routingregret}
\label{thm:main-regret}
Assume the bounded non-interventional protocol of \cref{sec:problem} and
\cref{ass:response-model,ass:risk-model}. Fix $T\ge1$, $\delta\in(0,1)$,
feature maps and positive-definite priors before interaction; public
feature vectors may depend on history. Run \ORUCB{} with exact constrained
response fits, full-history weighted queried-risk estimates, $0<\eta\le1/2$
and \cref{eq:main-calibration}. With probability at least $1-\delta$
over the joint interaction law,
\begin{equation}
\Reg_T\le C_{\rm prob}
\left\{\sqrt T\log\frac{T+1}{\delta}+\log\frac{T+1}{\delta}\right\},
\label{eq:main-rate}
\end{equation}
where $C_{\rm prob}$ is finite and depends only on the residual bounds,
dimensions, registry capacity, regularizers, priors, parameter-norm bounds
and pool rate, not on $T$ or $\delta$.
\end{restatable}
For fixed problem parameters and $\delta=(T+1)^{-2}$, the bound is
$O(\sqrt T\log(T+1))$. Bounded one-round regret also gives vanishing
expected average regret. Thus the cost of learning the corrections and
query decisions is sublinear relative to the oracle, even though only
purchased answers train the responses. Appendix~\ref{app:finite-horizon}
gives a fully explicit bound with sharper dependence on query counts.

\textbf{Proof idea.}
\Cref{lem:response-certificate} controls the total error from learning
responses. The calibrated weights prevent this error from invalidating
the risk confidence bounds. Optimistic acquisition then compares each
chosen action with the oracle; summing the estimation and response errors
gives the regret bound.
Appendix~\ref{app:optimal-risk-proof} gives the proof.

\textbf{Comparison with direct deferral.}
Along the same public history, the learner's cumulative conditional cost
minus that of the direct-deferral oracle equals $\Reg_T$ minus the
cumulative oracle gain from correction. On the theorem's event, the learner
therefore improves on direct deferral whenever this gain exceeds the
regret bound. If the average oracle gain stays bounded away from zero,
the linear cumulative benefit eventually exceeds the sublinear learning cost.
Appendix~\ref{app:routing-gain} characterizes when correction
produces such a gain.
For approximate response and risk models, Appendix~\ref{app:approximate-models}
bounds the additional regret using supplied approximation-error bounds.

%% file: sections/05_experiments.tex
\section{Experiments}
\label{sec:experiments}
\begingroup
\setlength{\intextsep}{6pt plus1pt minus1pt}
\setlength{\textfloatsep}{6pt plus1pt minus1pt}
\setlength{\abovecaptionskip}{4pt}
\setlength{\belowcaptionskip}{3pt}
We examine (i) the cost of learning to query and correct answers against
seven external baselines on four streams (\cref{tab:main-costs});
(ii) sensitivity to six consultation prices
(\cref{tab:main-price-comparison}); and (iii) routing with a common
response learner, together with learning and query-rate curves, in
Appendix~\ref{app:experiment-details}.

\textbf{Setting.}
Synthetic regimes control when free prediction, direct deferral, or correction is useful. Melbourne daily
minimum temperatures \citep{daily_min_temp_melbourne} pair real outcomes with
simulated sensors to control answer bias and noise. On Jena
\citep{jena_climate} and Delhi \citep{daily_delhi_climate}, experts are frozen
random forests, extra trees, and gradient-boosted trees; Jena also uses
radial-basis kernel ridge. Jena uses contemporaneous
weather covariates; Melbourne and Delhi use past observations and calendar
features. Availability schedules include expert absence and return.

We compare SleepingLinUCB and SharedLinUCB, our available-action LinUCB
adaptations \citep{li2010contextual,kleinberg2010sleeping};
D-LinUCB \citep{russac2019weighted}; LinTS \citep{agrawal2013thompson};
NeuralUCB \citep{zhou2020neuralucb}; and CUSUM-LinUCB and GLR-LinUCB,
linear adaptations of the change detectors of \citet{liu2018change} and
\citet{besson2022efficient}. For each method and dataset, we select from
$256$ configurations by validation cost and report fee-inclusive cost on
separate test blocks. Appendices~\ref{app:experimental-setup}
and~\ref{app:data-construction} describe evaluation, warmup, features, and experts.

\begin{table}[H]
\centering\small
\setlength{\tabcolsep}{3pt}
\caption{Mean test cost, including consultation fees. Entries show
mean $\pm$ sample standard deviation across five learner seeds. Lower is better.}
\label{tab:main-costs}
\resizebox{.96\linewidth}{!}{\input{output/pdf/final_weighted/tables/cost_table}}
\end{table}

\textbf{Results.}
\ORUCB{} has lower mean test cost than every external baseline on all
four streams. At benchmark fees, it queries on every scored round on
Jena and Delhi. The six-price Jena study then tests acquisition as fees
increase. Our method has the lowest cost at all six prices
(\cref{tab:main-price-comparison}), with equal search budgets per price.
Across these price-specific configurations, its query fraction falls from
$96\%$ to $0.1\%$, spanning frequent consultation to near-free prediction.

\begin{table}[H]
\centering\small
\setlength{\tabcolsep}{3pt}
\caption{Mean Jena test cost at six consultation prices. Columns scale the base fee
$(\approx 0.17)$. Entries show mean $\pm$ sample standard deviation over
five learner seeds.  Bold marks the lowest mean per column.
Fees and query percentages are in Appendix~\ref{app:fee-sensitivity}.}
\label{tab:main-price-comparison}
\resizebox{\linewidth}{!}{\input{tables/main_price_comparison}}
\end{table}

\textbf{Extra experiments.}
Appendix~\ref{app:benchmark-queries} gives full learning trajectories, and
Appendix~\ref{app:fee-sensitivity} reports cost and query-rate curves.
Appendix~\ref{app:jena-routing} holds the cubic correction learner,
pooling, and memory fixed while comparing joint and disjoint UCB routers.
Joint response UCB uses our risk features with unit-weight regression;
disjoint response UCB fits a separate risk model per expert. Each selects
its confidence and risk regularization from $256$ configurations per stream.
Our router matches joint response UCB on three streams and improves on
both alternatives on Delhi.
\FloatBarrier
\endgroup

%% file: output/pdf/final_weighted/tables/cost_table.tex
\begin{tabular}{lcccc}
\toprule
Method & Synthetic & Melbourne & Jena & Delhi \\
\midrule
\ORUCB{} & \textbf{0.08092} & \textbf{4.49069} & \textbf{3.72844} & \textbf{2.59240} \\
SleepingLinUCB & 0.08488 & 4.61400 & 4.03045 & 2.66401 \\
SharedLinUCB & 0.08316 & 4.90137 & 4.02217 & 2.64185 \\
D-LinUCB & 0.08790 & 5.11775 & 4.07844 & 2.63729 \\
LinTS & 0.08609 {\scriptsize$\pm$ 0.00389} & 5.42987 {\scriptsize$\pm$ 0.15762} & 4.11937 {\scriptsize$\pm$ 0.04932} & 2.74807 {\scriptsize$\pm$ 0.04935} \\
NeuralUCB & 0.10183 {\scriptsize$\pm$ 0.00604} & 5.58128 {\scriptsize$\pm$ 0.26329} & 4.18453 {\scriptsize$\pm$ 0.13384} & 2.92431 {\scriptsize$\pm$ 0.20347} \\
CUSUM-LinUCB & 0.10532 {\scriptsize$\pm$ 0.00107} & 5.61922 {\scriptsize$\pm$ 0.12936} & 4.33342 {\scriptsize$\pm$ 0.03940} & 2.79313 {\scriptsize$\pm$ 0.05511} \\
GLR-LinUCB & 0.10321 {\scriptsize$\pm$ 0.00141} & 5.24715 {\scriptsize$\pm$ 0.16709} & 4.34858 {\scriptsize$\pm$ 0.05603} & 2.91288 {\scriptsize$\pm$ 0.05895} \\
\bottomrule
\end{tabular}

%% file: tables/main_price_comparison.tex
\begin{tabular}{lcccccc}
\toprule
Method & $\times1$ & $\times3$ & $\times5$ & $\times7.5$ & $\times12.5$ & $\times50$ \\
\midrule
\ORUCB{} & \textbf{3.711} & \textbf{3.984} & \textbf{4.024} & \textbf{4.079} & \textbf{4.117} & \textbf{4.174} \\
SleepingLinUCB & 4.030 & 4.277 & 4.240 & 4.444 & 4.417 & 4.190 \\
SharedLinUCB & 4.022 & 4.201 & 4.180 & 4.242 & 4.202 & 4.182 \\
D-LinUCB & 4.078 & 4.371 & 4.447 & 4.452 & 4.485 & 4.552 \\
LinTS & 4.119 {\scriptsize$\pm$ 0.049} & 4.167 {\scriptsize$\pm$ 0.040} & 4.177 {\scriptsize$\pm$ 0.036} & 4.233 {\scriptsize$\pm$ 0.044} & 4.250 {\scriptsize$\pm$ 0.065} & 4.182 \\
NeuralUCB & 4.185 {\scriptsize$\pm$ 0.134} & 4.223 {\scriptsize$\pm$ 0.108} & 4.190 {\scriptsize$\pm$ 0.069} & 4.169 {\scriptsize$\pm$ 0.045} & 4.189 {\scriptsize$\pm$ 0.020} & 4.182 \\
CUSUM-LinUCB & 4.333 {\scriptsize$\pm$ 0.039} & 4.331 {\scriptsize$\pm$ 0.073} & 4.328 {\scriptsize$\pm$ 0.030} & 4.362 {\scriptsize$\pm$ 0.050} & 4.415 {\scriptsize$\pm$ 0.051} & 4.666 {\scriptsize$\pm$ 0.048} \\
GLR-LinUCB & 4.349 {\scriptsize$\pm$ 0.056} & 4.489 {\scriptsize$\pm$ 0.038} & 4.557 {\scriptsize$\pm$ 0.021} & 4.574 {\scriptsize$\pm$ 0.066} & 4.641 {\scriptsize$\pm$ 0.049} & 5.113 {\scriptsize$\pm$ 0.085} \\
\bottomrule
\end{tabular}

%% file: sections/06_conclusion.tex
\vfill
\section{Conclusion}
\label{sec:discussion}
We studied online learning to defer when a purchased answer can be
corrected before deployment. \ORUCB{} learns these corrections and accounts
for their learning error when deciding which answers to buy. Under the
stated response and risk models, the calibrated algorithm achieves
sublinear pseudo-regret against the conditional response oracle; an
extension quantifies the effect of approximation errors. On four test
streams, the selected cubic policy has lower cost than seven direct-deferral
baselines. With the correction learner held fixed, our router matches or
improves on the tested UCB alternatives. Six-price comparisons measure
how consultation cost and query rates vary with fees.
\par
\kern9.075pt

%% file: sections/07_ai_use.tex
\subsection*{AI use statement}
\label{sec:ai-use}
We used large language model assistants to polish the writing, check
mathematical arguments and proofs. We were assisted by AI to develop the nonlinear-correction
example in Appendix~\ref{app:orthogonal-separation}. We have also generated the synthetic data with AI and checked
very carefully its soundness. We also used these
tools to implement and debug methods and experimental code. We reviewed the AI-assisted material and take full
responsibility for the final text, mathematical claims, code, and results.

%% file: appendix/a_protocol_and_dimensions.tex
\section{Probability space and timing}
\label{app:protocol}
\subsection{Information and timing}
\Cref{fig:consultation-intuition} summarizes the timing described in
\cref{sec:problem}. The distinction between choosing a function and
evaluating it explains how a prediction can depend on a purchased answer
even though its response coefficients are fixed before purchase.
\begin{figure}[H]
\centering
\input{figures/consultation_intuition}
\caption{Consultation with a learned response, with round and expert
subscripts suppressed. The solid path shows our queried action; the dashed
branch shows standard deferral for comparison. The entire function $\Delta$
is fixed before observing $M$; its value is evaluated afterward.
Free prediction, omitted here, uses $h$ without paying a fee.}
\label{fig:consultation-intuition}
\end{figure}
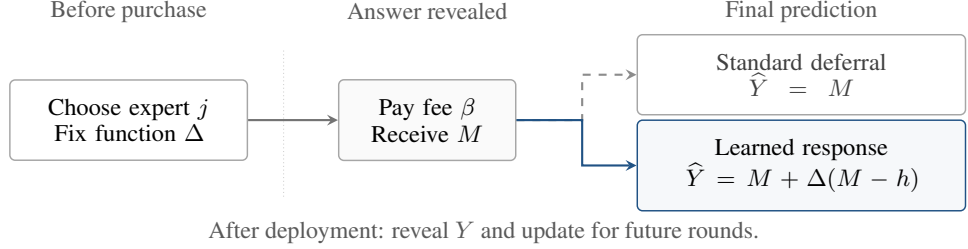

All random variables are defined on a common probability space admitting
regular conditional laws for the finite-dimensional observations. We choose
conditional versions simultaneously over the countable rounds and finite
registry. The pre-query field $\cF_t$ includes completed outcomes, purchased
answers, algorithm state, current public features, $h_t$, availability and
fees. The decision and response coefficients are $\cF_t$-measurable. For a
queried action, the intermediate field is
$\mathcal G_t=\cF_t\vee\sigma(M_{t,I_t})$; on a free round it is $\cF_t$.
The deployed prediction is $\mathcal G_t$-measurable, and $Y_t$ is revealed
afterward.
The next pre-query field contains $\mathcal G_t$ and $Y_t$. Thus a purchased
answer can enter subsequent response fits, while the current identity and
response function are chosen before that answer is observed.

Conditional on $\cF_t$, a common joint law specifies $Y_t$ and all potential
expert answers. The current action reveals coordinates of this vector
without changing its law. Neither the expert answers nor the rounds need to
be independent. An internal predictor may update causally on every outcome,
provided its current $h_t$ is fixed before routing and the stated model
assumptions continue to hold. A returning identity keeps its learned state;
the guarantee requires its conditional law on return to satisfy the same
fixed parameter models.

To define the conditional oracle, we need a measurable choice of response.
For fixed $t,j$, its objective is a continuous quadratic on the compact ball
$\mathbb B_q$, with coefficients given by measurable conditional moments.
The measurable maximum theorem \citep{aliprantis2006infinite} therefore gives
a measurable optimizer. Minimizing over a finite eligible set with a fixed tie
rule gives a measurable oracle action. We use these versions throughout the
proofs. Because rounds and identities are countable, the boundedness
assumptions may be imposed on a common probability-one event.

The fee screen preserves the oracle value. Since $0\le L_t^0\le R^2$,
every query with
$\beta_{t,j}\ge R^2$ has conditional cost at least $R^2$ even if its prediction
loss is zero. Excluding those actions preserves the oracle value, including
boundary ties. The learner consequently incurs at most $4G^2+R^2$ per round,
and its pseudo-regret is nonnegative because its deployed coefficient belongs
to the oracle's response ball.

\clearpage
\subsection{Notation guide}
The main notation is collected below. Public quantities, including the
features and response coefficients, are known given $\cF_t$; the outcome
and current expert answers are still hidden. Superscripts $H,P$ distinguish
the identity and shared response proposals. Superscripts $0,*$ distinguish
the free-risk and optimal queried-risk regressions.

\begin{table}[H]
\centering\small
\renewcommand{\arraystretch}{1.2}
\caption{Decisions, responses, and costs. All risks exclude consultation fees.}
\label{tab:notation-protocol}
\begin{tabular}{lp{.66\linewidth}}
\toprule
Symbol & Meaning\\
\midrule
$K,T$ & Registry capacity and interaction horizon.\\
$\cF_t$ & Information available before the round-$t$ query.\\
$\cE_t,\cA_t$ & Available identities and the eligible subset with fees below $R^2$.\\
$I_t,\beta_{t,j}$ & Chosen action ($0$ means free prediction) and the fee for expert $j$.\\
$h_t,M_{t,j},Y_t$ & Internal prediction, potential expert answer, and outcome.\\
$d_{t,j},e_{t,j}$ & Disagreement $M_{t,j}-h_t$ and expert residual $Y_t-M_{t,j}$.\\
$R,D,G$ & Bounds on $|Y_t-h_t|$, $|d_{t,j}|$, and $|e_{t,j}|$, with $G=R+D$.\\
$\phi_q,\mathbb B_q$ & Dictionary in $\R^{q+1}$ and its coefficient unit ball.\\
$a^H_{t,j},a^P_{t,j},a_{t,j}$ & Identity proposal, shared proposal, and their deployed mixture.\\
$\widehat Y_t,C_t$ & Deployed prediction and realized loss plus fee.\\
$L_t^0,W_{t,j}^*$ & Free risk and minimum queried risk over $\mathbb B_q$.\\
$C_t^*,\Reg_T$ & Oracle conditional cost and cumulative conditional pseudo-regret.\\
\bottomrule
\end{tabular}
\end{table}

\begin{table}[H]
\centering\small
\renewcommand{\arraystretch}{1.2}
\caption{Learning state and confidence quantities. Dimensions count feature coordinates.}
\label{tab:notation-learning}
\begin{tabular}{lp{.62\linewidth}}
\toprule
Symbol & Meaning\\
\midrule
$\ell_{t,j}\in\R^{d_\ell},\ s_{t,j}\in\R^{d_s}$ & Public features for identity and shared responses.\\
$\Theta_{t,j},\Psi_t$ & Response matrices of sizes $d_\ell\times(q+1)$ and $d_s\times(q+1)$.\\
$w^H_{t,j},w^P_{t,j},\eta$ & Proposal mixture weights and their exponential-weights rate.\\
$x_t^0\in\R^{d_0},\ c_{t,j}\in\R^{d_c}$ & Public features for free and optimal queried risk.\\
$\Lambda_v,V_t^v,\widehat\theta_t^v$ & Prior, pre-round design, and risk estimate, for $v\in\{0,*\}$.\\
$\lambda_v,Q_v,B_v$ & Minimum prior eigenvalue, prior-norm bound, and noise-plus-prior radius.\\
$u_{t,j},\omega_t,\tau$ & Queried-risk width, regression sample weight, and its threshold.\\
$\alpha,S_{t,j}$ & Confidence multiplier and acquisition score.\\
$Z_t,\rho_t$ & Queried loss divided by $4G^2$ and response excess risk divided by $G^2$.\\
$N_{t,j},N_t$ & Queries through round $t$ to identity $j$, and in total.\\
$A_t,E_t,\overline E_T$ & Realized response-loss budget, conditional-risk budget, and deterministic horizon envelope.\\
\bottomrule
\end{tabular}
\end{table}
In particular, $w^H,w^P$ combine responses, while $\omega_t$ weights a
risk-regression observation. Designs $V_t^v$ use rounds strictly before
$t$, whereas $N_{t,j}$ includes round $t$. We write $\mathrm{Id}_r$ for
the $r\times r$ identity matrix to distinguish it from the chosen action $I_t$.

%% file: figures/consultation_intuition.tex
\begin{tikzpicture}[
  x=1cm,y=1cm, font=\small,
  box/.style={draw=black!35, line width=.55pt, rounded corners=2pt,
    align=center, inner xsep=7pt, inner ysep=7pt},
  flow/.style={->, >=stealth, line width=.8pt, draw=black!55},
  stage/.style={font=\small, text=black!65, align=center}
]
\node[stage] at (1.65,.16) {Before purchase};
\node[stage] at (5.6,.16) {Answer revealed};
\node[stage] at (10.55,.16) {Final prediction};

\node[box, text width=2.65cm] (decision) at (1.65,-1.28)
  {Choose expert $j$\\Fix function $\Delta$};
\node[box, text width=1.85cm, fill=black!2] (answer) at (5.6,-1.28)
  {Pay fee $\beta$\\Receive $M$};
\node[box, text width=3.85cm, text=black!75] (direct) at (10.55,-.68)
  {Standard deferral\\$\widehat Y=M$};
\node[box, text width=3.85cm, draw=citationblue, fill=citationblue!4]
  (response) at (10.55,-1.88)
  {Learned response\\$\widehat Y=M+\Delta(M-h)$};

\draw[flow] (decision.east) -- (answer.west);
\draw[flow, dashed, draw=black!45] (answer.east) -- (7.65,-1.28)
  |- (direct.west);
\draw[flow, draw=citationblue] (answer.east) -- (7.65,-1.28)
  |- (response.west);
\draw[black!18, densely dotted] (3.70,-.36) -- (3.70,-2.32);
\node[font=\small, text=black!65] at (6.35,-2.76)
  {After deployment: reveal $Y$ and update for future rounds.};
\end{tikzpicture}

%% file: appendix/e_implementation.tex
\section{Implementation and explicit regret bound}
\label{app:implementation}
This appendix gives the response budget and recursive updates for the
identity model, followed by the connection between our assumptions and
online-learning models, a continuous-context example, and the explicit
finite-horizon regret bound. The proof of that
bound appears in Appendix~\ref{app:optimal-risk-proof}.

\begin{algorithm}[H]
\caption{\ORUCB{}: online updates}
\label{alg:optimal-risk}
\begin{algorithmic}[1]
\Require Priors, response regularizers, $0<\eta\le1/2$,
score multiplier $\alpha\ge0$, and weight threshold $\tau>0$.
\State Initialize all regression sums and query counts to zero, designs to their priors, and each identity's pool weights to $(1/2,1/2)$. Preserve all states during absence.
\For{$t=1,\ldots,T$}
\State Observe $h_t,\cE_t$, fees and public features; register new identities and restore returning states.
\State Set $\cA_t=\{j\in\cE_t:\beta_{t,j}<R^2\}$. Compute the ridge forecasts in \cref{eq:risk-estimators}.
\State For each $j\in\cA_t$, freeze $a^H_{t,j}$ from \cref{eq:identity-response}, $a^P_{t,j}$ from \cref{eq:population-response,eq:population-projection}, and $a_{t,j}$ from \cref{eq:pool}.
\State Compute \cref{eq:score}; set $I_t$ to its first maximizer in the order $(0,\cA_t)$, with experts ordered by registration.
\If{$I_t=0$}
\State Deploy $\widehat Y_t=h_t$.
\Else
\State Freeze $\omega_t$ by \cref{eq:query-weight} with the chosen $\tau$.
\State Pay $\beta_{t,I_t}$, observe $M_{t,I_t}$, and deploy $\widehat Y_t=M_{t,I_t}+Ga_{t,I_t}^\top\phi_q(M_{t,I_t}-h_t)$.
\EndIf
\State Observe $Y_t$; update free ridge with $(x_t^0,(Y_t-h_t)^2/R^2)$.
\If{$I_t\ne0$}
\State Form $Z_t$ from the deployed prediction; add $\omega_tc_{t,I_t}c_{t,I_t}^\top$ to the queried design and $\omega_tc_{t,I_t}Z_t$ to its cross-product.
\State Add the purchased observation to the two response objectives \cref{eq:identity-response,eq:population-response}; update the pool using the frozen proposal losses and increment $N_{t,I_t}$.
\EndIf
\EndFor
\end{algorithmic}
\end{algorithm}

\subsection{The response-learning bound}
\label{app:response-budget}
The bound $E_t$ in \cref{eq:response-certificate-main} accounts for the
loss accumulated while fitting and combining responses. Let
$p_I=d_\ell(q+1)$ be the number of entries in each identity's response matrix.
The constrained least-squares comparison in \cref{lem:constrained-regret}
contributes a regularization term and a logarithmic term for each identity.
The exponential-weights comparison in \cref{lem:pool} adds $4\log2/\eta$
per identity. Their sum is the pathwise bound
\begin{equation}
A_t=K\lambda_I+8p_I\sum_{j=1}^K
\log\left(1+\frac{N_{t,j}}{p_I\lambda_I}\right)
+\frac{4K\log2}{\eta}.
\label{eq:identity-pathwise-budget}
\end{equation}
This is a bound on the cumulative realized loss of the deployed response
relative to the optimal identity response, in normalized units. It is a
function of the observed query counts, rather than a deterministic constant. The
normalization divides
squared losses by $G^2$; the comparison uses the fixed optimal matrices in
\cref{ass:response-model}.
Here $N_{t,j}$ counts queries through round $t$, with all counts initially
zero, and $N_t=\sum_jN_{t,j}$ counts all purchases.
The calibrated router uses the deterministic horizon envelope below
to fix its weighting threshold before interaction.

\Cref{lem:response-conversion} converts this realized-loss bound into a
conditional excess-risk bound. Define
\begin{equation}
E_t=2A_t+64\log(3/\delta).
\label{eq:identity-envelope}
\end{equation}
With probability at least $1-\delta/3$, the inequality
$\sum_{s\le t}\rho_s\le E_t$ holds simultaneously over all prefixes.
Concavity of the logarithm and $\sum_jN_{t,j}\le T$ give the deterministic
horizon envelope
\begin{equation}
\overline E_T=2\left[K\lambda_I+
8Kp_I\log\left(1+\frac{T}{Kp_I\lambda_I}\right)
+\frac{4K\log2}{\eta}\right]+64\log\frac3\delta,
\qquad E_t\le\overline E_T\quad(t\le T).
\end{equation}
This is the envelope used to set the weight threshold in
\cref{eq:main-calibration}.
The factor $3$ reserves the other two parts of the failure probability for
the free and queried risk regressions.

\subsection{Risk regressions and confidence calibration}
In the benchmarks, the free, identity-response and shared-response features
use the same bounded public state: $x_t^0=\ell_{t,j}=s_{t,j}$ for every
active expert $j$.
The queried-risk fit shares information across experts while retaining
separate identity coordinates. Its feature in the implementation is
$c_{t,j}=2^{-1/2}(s_{t,j},\boldsymbol e_j\otimes\ell_{t,j})$, where
the vectors are stacked, $\boldsymbol e_j\in\R^K$ is the registry basis
vector, and $\otimes$ denotes the Kronecker product. Thus $d_c=d_s+Kd_\ell$.
Write $\mathrm{Id}_r$ for the $r\times r$ identity matrix, reserving $I_t$
for the selected action. For strengths
$\lambda_s,\lambda_i>0$, the block-diagonal prior
$\Lambda_*=\diag(\lambda_s\mathrm{Id}_{d_s},\lambda_i\mathrm{Id}_{Kd_\ell})$ controls shrinkage
separately on shared and identity coordinates.
Both priors are fixed before interaction. The full-history estimates in
\cref{eq:risk-estimators} use only observations available before round $t$.
Let $\clip_{[a,b]}(v)=\min\{b,\max\{a,v\}\}$ denote scalar clipping.
Use $\widehat L_t^0=R^2\clip_{[0,1]}((x_t^0)^\top\widehat\theta_t^0)$ and
the raw forecast $\widetilde W_{t,j}=4G^2c_{t,j}^\top\widehat\theta_t^*$.
Clipping the free forecast cannot increase its distance from the true free
risk in $[0,R^2]$. We retain the raw queried forecast when applying the
confidence inequality that accounts for response-learning bias.

Let $Q_v$ bound the corresponding parameter norm in
\cref{ass:risk-model}, measured by $\|\cdot\|_{\Lambda_v}$.
For $v\in\{0,*\}$, put $\lambda_v=\lambda_{\min}(\Lambda_v)>0$ and
$d_*=d_c$. A valid horizon calibration is
\begin{equation}
B_v=\tfrac12\sqrt{d_v\log(1+T/(d_v\lambda_v))+2\log(3/\delta)}+Q_v,
\qquad\alpha=2B_*,\quad\tau=4B_*/\overline E_T.
\label{eq:critic-calibration}
\end{equation}
The explicit regret bound below uses this coupled calibration.
One may take $Q_v=\sqrt{\lambda_{\max}(\Lambda_v)}S_v$ from a known
Euclidean norm bound $S_v$. More precisely, the self-normalized radius uses
$\frac12\sqrt{\log(\det V_t^v/\det\Lambda_v)+2\log(3/\delta)}+Q_v$.
The horizon calibration bounds this radius throughout the analyzed rounds.

\subsection{Computing the response fits}
The identity response is obtained by solving a quadratic problem on a unit
ball. For identity $j$, vectorize its response matrix row-wise and write
$z_s=\ell_{s,j}\otimes\phi_q(d_{s,j})\in\R^{p_I}$ and $y_s=e_{s,j}/G$
on rounds with $I_s=j$. Its pre-query design and cross-product vector are
\[
V=\lambda_I\mathrm{Id}_{p_I}+\sum_{s<t:I_s=j}z_sz_s^\top,
\qquad b=\sum_{s<t:I_s=j}z_sy_s.
\]
The constrained objective is a positive-definite quadratic on the Euclidean
unit ball. The quadratic optimality conditions \citep{boyd2004convex} imply
that $V^{-1}b$ is the solution whenever it is feasible. Otherwise the solution is
$(V+\mu\mathrm{Id}_{p_I})^{-1}b$, where $\mu>0$ is the unique root of
$\|(V+\mu\mathrm{Id}_{p_I})^{-1}b\|_2=1$. After diagonalization, the norm decreases
monotonically with the multiplier, so a scalar root search finds the solution.
We compute the constrained fit using 90 bisection iterations and check
coefficient feasibility with tolerance $10^{-10}$. The analysis uses exact
minimizers.

The shared matrix $\Psi_t$ solves the unconstrained ridge problem in
\cref{eq:population-response}; its proposal is projected by
\cref{eq:population-projection}. This regression uses the product feature
$s_{r,I_r}\otimes\phi_q(d_{r,I_r})$ with row-wise vectorization of $\Psi_t$.
It incorporates the observed answer and outcome only after deployment.
The proposal remains predictable and bounded even if the shared model is
misspecified. The queried label is computed from the coefficient average
fixed before querying; the later fit updates do not change that label.

\subsection{Recursive implementation of the deployed policy}
\label{app:recursive-algorithm}
All quantities below use the learner's normalized units; the dataset-specific
rescaling of observations, bounds, fees and reported costs is specified in
Appendix~\ref{app:state-construction}. In empirical runs,
$\alpha_{\rm emp},\tau_{\rm emp}$ denote the selected values of
$\alpha,\tau$. Set $p_P=d_s(q+1)$, the shared response dimension.
A scalar-target ridge state consists of an inverse design $P\in\R^{p\times p}$
and a cross-product vector $b\in\R^p$. Given a feature $z\in\R^p$ and target
$y\in\R$, define the update $\mathsf U(P,b;z,y)$ by
\begin{equation}
 v=Pz,\qquad
 P^+=P-\frac{vv^\top}{1+z^\top v},\qquad b^+=b+zy.
\label{eq:recursive-ridge}
\end{equation}
The implementation symmetrizes $P^+$ after each update. The prediction is
$z^\top Pb$ and its uncertainty width is $\sqrt{z^\top Pz}$.
A negative squared width within the implementation tolerance $10^{-10}$ is
set to zero; a larger negative value raises an error. These rank-one updates
implement the
full-history objectives above in exact arithmetic.

Initialize the free and queried states as
$(P_0,b_0)=(\Lambda_0^{-1},0)$ and
$(P_*,b_*)=(\Lambda_*^{-1},0)$, respectively, and the shared response state as
$(P_P,b_P)=(\lambda_P^{-1}\mathrm{Id}_{p_P},0)$.
The implemented free prior is $\Lambda_0=\lambda_0\mathrm{Id}_{d_0}$, with
$\lambda_0>0$; the queried prior has the shared and identity blocks given above.
For each newly registered identity $j$, initialize
$(P_j,b_j)=(\lambda_I^{-1}\mathrm{Id}_{p_I},0)$, count $N_j=0$, and cumulative proposal
losses $H_j^H=H_j^P=0$. New identities are registered in their supplied public
order; existing registration ranks never change. Absence changes none of
these states. The registry has capacity $K$.

Before querying, reshape the constrained solution from $(P_j,b_j)$ row-wise
into $\Theta_{t,j}$, and reshape $P_Pb_P$ row-wise into $\Psi_t$.
Use \cref{eq:identity-response,eq:population-projection} for the proposals.
Their weights are proportional to $\exp(-\eta H_j^u)$ for $u\in\{H,P\}$;
in code, subtracting the largest log weight before exponentiation stabilizes
this calculation. The pool contains the shared and identity proposals, both
with initial weight $1/2$, and its coefficient is always deployed on a
query. The queried-risk width is computed from $P_*$ before acquisition.
After selecting a queried identity, freeze its weight
$\omega_t=\min\{1,\tau_{\rm emp}/\sqrt{c_{t,j}^\top P_*c_{t,j}}\}$,
using one at zero width. Neither the answer nor the outcome enters this
weight. Pool weights $w^H,w^P$ and the regression weight $\omega_t$ have
different roles.
The queried forecast is raw, the free forecast is clipped, and there is no
free-risk optimism term. Free prediction has priority on score ties, followed
by the earliest registered eligible identity.

After observing the outcome, update
$(P_0,b_0)\leftarrow\mathsf U(P_0,b_0;x_t^0,(Y_t-h_t)^2/R^2)$ on every round.
If $j=I_t\ne0$, retain the pre-query proposals and public features, and put
$d=M_{t,j}-h_t$, $y=(Y_t-M_{t,j})/G$, and $\phi=\phi_q(d)$.
Compute $Z_t$ from the prediction already deployed, then apply
\begin{equation}
\begin{aligned}
(P_*,b_*)&\leftarrow\mathsf U(P_*,b_*;\sqrt{\omega_t}c_{t,j},\sqrt{\omega_t}Z_t),\\
(P_j,b_j)&\leftarrow\mathsf U(P_j,b_j;\ell_{t,j}\otimes\phi,y),\\
(P_P,b_P)&\leftarrow\mathsf U(P_P,b_P;s_{t,j}\otimes\phi,y),\qquad
N_j\leftarrow N_j+1.
\end{aligned}
\label{eq:recursive-query}
\end{equation}
For each frozen proposal, update
\begin{equation}
H_j^u\leftarrow H_j^u+\mathcal L_t^u,\qquad
\mathcal L_t^u=\tfrac14\bigl(y-(a^u_{t,j})^\top\phi\bigr)^2,
\qquad u\in\{H,P\}.
\label{eq:recursive-weights}
\end{equation}
Equivalently, multiply its current weight by $\exp(-\eta\mathcal L_t^u)$
and normalize the two weights. All targets and proposal losses use the
observations and proposals already
specified for the round. These state updates therefore give the same
next-round mathematical state in any order. Other identities retain
their fits, counts and weights. Free rounds update none of the queried or
response states. There is no clipping of the final corrected prediction. Observed residuals
and disagreements are checked against the declared bounds, permitting only
one floating-point unit beyond the boundary; a violation raises an error.

Both feature and target are scaled by $\sqrt{\omega_t}$, so the queried
cross-product increment is $\omega_tc_{t,j}Z_t$. The response fits retain
their unweighted updates.
The theoretical run uses \cref{eq:main-calibration} for these two constants
and exact constrained minimization. The benchmark run uses the numerical
solver specified above.

\subsection{Model assumptions in context}
\label{app:assumption-context}
Our assumptions connect two learned objects to the response oracle:
the correction coefficients and the risk attained by an optimal correction.
Their roles follow established online-regression and linear-bandit analyses.

\textbf{Optimal-response realizability.}
Realizability connects a learned function class to an oracle decision rule.
For example, \citet[Assumption~1]{foster2020beyond} assume that one function
in the regression class represents conditional mean costs, and their
Assumption~2a compares online squared loss with a fixed function on
adaptively chosen observations. Our \cref{ass:response-model} instead
represents a conditional risk minimizer: one feasible matrix per identity
produces optimal coefficients across contexts. It lets us compare the
online response fit with the oracle using the fixed-comparator principle
of online convex optimization \citep[Chapter~3]{hazan2023online}.
The exact constrained least-squares bound used here is proved in
\cref{lem:constrained-regret}. The optimal correction may be nonunique or
lie on the coefficient-ball boundary, and need not recover the conditional
mean residual. The shared model may be misspecified; its projected
response supplies a feasible auxiliary proposal.

\textbf{Linear risk models.}
\citet[Theorem~2]{abbasi2011improved} construct confidence ellipsoids for
a fixed linear conditional mean under a parameter-norm bound and
conditionally sub-Gaussian noise. Bounded features and parameter norms
also appear in \citet[Section~2.1]{agrawal2013thompson} and
\citet[Assumption~2.1]{he2022nearly}. These are the conditions used for
our risk regressions, with the optimal queried risk as their latent target.
Modeling that target separately from the coefficients specifies both
what the learner should deploy and what it should expect to pay.
For a general positive-definite prior, a Euclidean norm bound $S_v$
gives $Q_v=\sqrt{\lambda_{\max}(\Lambda_v)}S_v$, as used in
\cref{eq:critic-calibration}.

\textbf{Bounded observations and adaptive contexts.}
Our residual and coefficient bounds put the normalized squared-loss labels
in $[0,1]$, the cost range also used by \citet[Section~1.1]{foster2020beyond}.
Centering these labels conditionally gives $1/2$-sub-Gaussian noise by
Hoeffding's lemma. Thus the noise condition for self-normalized regression
follows from the protocol rather than a Gaussian answer model.
Predictable features can depend on previous observations, as in
\citet[Theorem~1]{abbasi2011improved} and
\citet[Section~2.1]{agrawal2013thompson}; independent contexts are not
needed. In our queried regression, the remaining conditional bias is the
response-learning gap in \cref{eq:upward-bias}. We bound its cumulative
size and its contribution to the confidence radius in
Appendix~\ref{app:risk-confidence}.

\subsection{A continuous-context example}
\label{app:contextual-example}

The following construction satisfies both main model assumptions with
continuously varying contexts.
Let $x_t\in[-1,1]$ be public and set
$\ell_{t,j}=(1,x_t)^\top/\sqrt2$. Conditional on $\cF_t$, let
$Y_t$ be uniform on $\{-0.3,0.3\}$, with internal prediction $h_t=0$ and
expert answer $M_{t,j}=Y_t+0.2+0.1x_t$.
Take $q=3$, $R=0.3$, $D=0.7$ and $G=1$. The correction
$\Delta(d)=-0.2-0.1x_t$ attains zero squared loss and is represented by
\[
\Theta_j^*=-\sqrt{2(1+3D^2)}
\begin{pmatrix}0.2&0&0&0\\0.1&0&0&0\end{pmatrix}.
\]
Its squared Frobenius norm is $2(1+3D^2)(0.2^2+0.1^2)=0.247<1$,
so it satisfies \cref{ass:response-model} over the full interval of contexts.
The free risk is $0.09$ and optimal queried risk is zero. For example, choose
$x_t^0=1$ and $s_{t,j}=\ell_{t,j}$, with the queried feature defined above.
Then $\theta_0^*=1$ and $\theta_*=0$ give the normalized risk models in
\cref{ass:risk-model}, with finite norms for any fixed positive-definite priors.
The optimal correction varies with context while the risks remain constant.

\subsection{Explicit finite-horizon bound}
\label{app:finite-horizon}
The regret bound combines uncertainty in the free-risk fit, uncertainty in
the queried-risk fit, and the loss from learning responses. For an observation
count $n\ge0$, a positive integer dimension $d_V$, and a prior eigenvalue
bound $\lambda>0$, define the cumulative-width bound
\begin{equation}
\mathcal W(n,d_V,\lambda)=
\sqrt{n(1+\lambda^{-1})d_V\log(1+n/(d_V\lambda))},
\qquad\mathcal W(0,d_V,\lambda)=0.
\label{eq:width-envelope}
\end{equation}
Define also
\begin{equation}
\mathcal H(n)=(1+\lambda_*^{-1})d_c\log(1+n/(d_c\lambda_*)),
\qquad\mathcal H(0)=0.
\label{eq:main-widths}
\end{equation}
\begin{restatable}[Explicit routing bound]{proposition}{explicitroutingbound}
\label{prop:explicit-routing-bound}
Under the hypotheses and calibration of \cref{thm:main-regret},
with probability at least $1-\delta$,
\begin{equation}
\begin{split}
\Reg_T\le{}&R^2B_0\mathcal W(T,d_0,\lambda_0)\\
&+G^2\left[16B_*\sqrt{N_T\mathcal H(N_T)}
+4\overline E_T\mathcal H(N_T)+E_T\right].
\end{split}
\label{eq:explicit-main-bound}
\end{equation}
\end{restatable}
This bound retains the separate contributions of free-risk estimation,
queried-risk estimation, and response learning. In particular, the
queried-risk terms depend on the number of purchases, not all $T$ rounds.
Appendix~\ref{app:optimal-risk-proof} restates and proves this bound together
with \cref{thm:main-regret}, including the reduction to its stated rate.

%% file: appendix/b_response_geometry.tex
\section{Response value and approximation}
\label{app:response-geometry}
A response can remove the part of an expert's error that its answer predicts.
We first describe this reduction in prediction loss, then account for fees,
and finally quantify the extra gain available from cubic terms.

\subsection{Prediction gain and approximation}
To quantify the prediction gain, fix a round $t$ and expert $j$, and abbreviate
$e=e_{t,j}$ and
$d=d_{t,j}$. Write $\nu=\E[e\mid\cF_t,d]$.
For a square-integrable response $\Delta(d)$
whose map is fixed under $\cF_t$, the squared-loss projection identity
\citep{hastie2009elements} gives
\begin{equation}
\E[(e-\Delta(d))^2\mid\cF_t]
=\E[\operatorname{Var}(e\mid\cF_t,d)\mid\cF_t]
 +\E[(\nu-\Delta(d))^2\mid\cF_t].
\label{eq:response-value}
\end{equation}
The first term is irreducible after observing the purchased answer; the
second measures how well the chosen correction approximates the conditional
mean residual. The best unrestricted correction is therefore $\nu$, up to
conditional $L_2$ equality. For direct deployment, $\Delta=0$, so the
removable part of its risk is $\E[\nu^2\mid\cF_t]$.

\medskip\noindent\textbf{Comparison with an unrestricted response.}
Restore the indices on
$\nu_{t,j}=\E[e_{t,j}\mid\cF_t,d_{t,j}]$. Define the best conditional
approximation error $\mathfrak A_{t,j}$ by
\[
\mathfrak A_{t,j}=\inf_{a\in\mathbb B_q}
 \E[(\nu_{t,j}-Ga^\top\phi_q(d_{t,j}))^2\mid\cF_t].
\]
Let $C_t^{\rm meas}$ be the oracle value with the same free action,
eligible experts and fees, but with any square-integrable response to the
purchased answer. The conditional mean response attains each queried
minimum, so a minimum over the finite action set is attained; denote an
optimizing action by $j_t^{\rm meas}$. Set $\mathfrak A_{t,0}=0$ for the free action.
By \cref{eq:response-value}, the bounded class adds
$\mathfrak A_{t,j}$ to the unrestricted risk of expert $j$. Evaluating it at
$j_t^{\rm meas}$, and using inclusion of the response classes, gives
\begin{equation}
0\le C_t^*-C_t^{\rm meas}\le\mathfrak A_{t,j_t^{\rm meas}}.
\label{eq:approximation-transfer}
\end{equation}
Summing and using the definition of $\Reg_T$ gives
\[
\sum_{t=1}^T\{\E[C_t\mid\cF_t]-C_t^{\rm meas}\}
\le\Reg_T+\sum_{t=1}^T\mathfrak A_{t,j_t^{\rm meas}}.
\]
The bound separates online learning error within the chosen response class
from its approximation error relative to the conditional mean residual.

\subsection{When correction changes the acquisition decision}
\label{app:routing-gain}
The projection identity quantifies the prediction error a response can
remove. To account for acquisition, this gain must be compared with the
expert's fee and the other available actions.
Use simultaneous conditional versions as in Appendix~\ref{app:protocol}.
For a fixed round define
\[
L_{t,j}^{\rm dir}=\E[e_{t,j}^2\mid\cF_t],\qquad
C_t^{\rm dir}=\min\{L_t^0,\min_{j\in\cA_t}
(L_{t,j}^{\rm dir}+\beta_{t,j})\},
\]
and, for every eligible expert,
\[
\Gamma_{t,j}=L_{t,j}^{\rm dir}-W_{t,j}^*,\qquad
\kappa_{t,j}=L_{t,j}^{\rm dir}+\beta_{t,j}-C_t^{\rm dir}.
\]
The gain $\Gamma_{t,j}$ measures the prediction loss removed by correction;
$\kappa_{t,j}$ is the expert's direct-cost disadvantage relative to the best
available direct action. Both are nonnegative. The following identity
compares these quantities without response or risk realizability.

\begin{equation}
C_t^{\rm dir}-C_t^*
=\max\bigl(\{0\}\cup
\{\Gamma_{t,j}-\kappa_{t,j}:j\in\cA_t\}\bigr).
\label{eq:response-routing-gain}
\end{equation}
Indeed, $W_{t,j}^*+\beta_{t,j}
=C_t^{\rm dir}+\kappa_{t,j}-\Gamma_{t,j}$.
Substituting into the minimum defining $C_t^*$ gives the identity:
the zero term covers a free direct optimum, and otherwise a queried
direct optimum has $\kappa_{t,j}=0$ and $\Gamma_{t,j}\ge0$.
An empty eligible set gives zero on both sides.

Correction strictly improves the oracle cost exactly when its prediction
gain exceeds some expert's direct-cost disadvantage.
Along the learner's history, the cumulative oracle advantage is
$\sum_t(C_t^{\rm dir}-C_t^*)$; the learner's conditional cost advantage
subtracts $\Reg_T$ from that sum.
The comparison uses the same realized public history for both costs.

\subsection{An example of nonlinear correction}
\label{app:orthogonal-separation}
A short example isolates the error that an affine response cannot remove.
Let $S\sim\operatorname{Unif}[-1,1]$, $h=0$, and let one always-available
expert return $M=S$ at fee $0<\beta<1/3$. For
$b,\varepsilon\in(0,1/4]$, take
$Y=S+b+\varepsilon P_3(S)$, $D=1$, $R=1+b+\varepsilon$ and $G=R+D$.
The public context is constant, so the answer is the only source of
information about the current draw.

Legendre orthogonality \citep{szego1975orthogonal} gives
$\E[S^2]=1/3$, $\E[P_3(S)^2]=1/7$ and $\E[SP_3(S)]=0$.
Both $S$ and $P_3(S)$ are centered. Thus an affine response removes the
constant bias $b$ but leaves risk $\varepsilon^2/7$, while a cubic
response removes both components. Writing $L_q^*$ for the optimal
one-round cost with order $q$ and $L_{\rm hard}^*$ for direct deferral,
\begin{equation}
L_3^*=\beta,\qquad L_1^*=\beta+\varepsilon^2/7,\qquad
L_{\rm hard}^*=\beta+\varepsilon^2/7+b^2.
\label{eq:orthogonal-costs}
\end{equation}
The optimal corrections are feasible: their coefficient vectors are
$a_1^*=(\sqrt2\,b/G,0)^\top$ and
$a_3^*=(2/G)(b,0,0,\varepsilon)^\top$, respectively, each with norm
strictly below one. Free prediction has risk
$1/3+b^2+\varepsilon^2/7$, so $\beta<1/3$ makes a query preferable in
every class. Pre-query randomization independent of $S$ cannot improve
these minima.

This example separates bias correction from nonlinear correction.
It does not rely on nesting the normalized coefficient balls:
the two coefficient vectors verify feasibility in their respective classes.

%% file: appendix/d_optimal_risk_proof.tex
\section{Proof of the routing guarantee}
\label{app:optimal-risk-proof}
The proof has four steps. We bound realized response regret on purchased
observations and convert it into conditional excess risk. That excess risk
controls the bias in the queried-risk regression. Confidence bounds then
support a one-round comparison with the oracle.
All probability statements use the filtration in Appendix~\ref{app:protocol}.
We analyze the identity-response model in \cref{ass:response-model}
throughout. The budgets $A_t,E_t$ and query counts $N_{t,j},N_t$ are those
in Appendix~\ref{app:response-budget}.

\subsection{Constrained least squares on selected observations}
The first step compares the response fit with a fixed feasible parameter on
exactly the observations used for fitting. We follow the stability analysis of
regularized leaders \citep{hazan2023online}, retaining the prior and query
counts in the bound.
Within the following lemma, $n$ indexes observations supplied to the fit,
rather than calendar rounds.
\begin{restatable}[Deterministic constrained-regression comparison]{lemma}{constrainedregression}
\label{lem:constrained-regret}
Let $z_n\in\R^{p}$ and $y_n\in\R$ be a sequence of features and
targets with $\|z_n\|_2\le1$ and $|y_n|\le1$. Let
$\mathcal C\subset\R^{p}$ be compact and convex, contain zero, and
satisfy $|z_n^\top\theta|\le1$ for every $\theta\in\mathcal C$ and $n$.
Let $\Lambda\in\R^{p\times p}$ be positive definite.
For exact minimizers
\[
\theta_n=\argmin_{\theta\in\mathcal C}
\{\|\theta\|_{\Lambda}^2+\sum_{i<n}(y_i-z_i^\top\theta)^2\},
\]
put $V_0=\Lambda$ and $V_n=\Lambda+\sum_{i=1}^nz_iz_i^\top$ for
$n\ge1$. For every
$\theta^*\in\mathcal C$ and every finite $N$, pathwise,
\begin{equation}
\sum_{n=1}^N\{(y_n-z_n^\top\theta_n)^2-(y_n-z_n^\top\theta^*)^2\}
\le\|\theta^*\|_{\Lambda}^2
 +8\log\frac{\det V_N}{\det\Lambda}.
\label{eq:constrained-comparison}
\end{equation}
\end{restatable}
\begin{proof}
We separate the comparison into the cost of regularization and the change
in loss caused by one update. Write $l_n(\theta)=(y_n-z_n^\top\theta)^2$ and
$F_n(\theta)=\|\theta\|_{\Lambda}^2+\sum_{i\le n}l_i(\theta)$.

\medskip\noindent\textbf{Comparator cost.}
Optimality of $\theta_n$ for $F_{n-1}$ gives
\[
F_n(\theta_{n+1})-F_{n-1}(\theta_n)\ge l_n(\theta_{n+1}).
\]
Summing and using $F_0(\theta_1)=0$, since zero is feasible, gives
$\sum_{n\le N}l_n(\theta_{n+1})\le F_N(\theta^*)$.
It remains to bound the loss change
$\sum_n\{l_n(\theta_n)-l_n(\theta_{n+1})\}$ caused by the updates.

\medskip\noindent\textbf{Update stability.}
Put $\Delta_n=\theta_n-\theta_{n+1}$ and $p_n=z_n^\top\theta_n$.
The variational inequalities for the two constrained minimizers are
\[
\nabla F_{n-1}(\theta_n)^\top\Delta_n\le0,
\qquad
\nabla F_n(\theta_{n+1})^\top\Delta_n\ge0.
\]
Since $\nabla F_n(\theta_n)=\nabla F_{n-1}(\theta_n)+2(p_n-y_n)z_n$
and the Hessian of $F_n$ is $2V_n$, subtraction gives
\[
\|\Delta_n\|_{V_n}^2\le(p_n-y_n)z_n^\top\Delta_n
\le2|z_n^\top\Delta_n|.
\]
Cauchy--Schwarz implies
$\|\Delta_n\|_{V_n}\le2\sqrt{z_n^\top V_n^{-1}z_n}$, including
$\Delta_n=0$. Consequently
\begin{align*}
l_n(\theta_n)-l_n(\theta_{n+1})
&=2(p_n-y_n)z_n^\top\Delta_n-(z_n^\top\Delta_n)^2\\
&\le4|z_n^\top\Delta_n|
\le8z_n^\top V_n^{-1}z_n.
\end{align*}
With $x_n=z_n^\top V_{n-1}^{-1}z_n$, the matrix inversion lemma gives
$z_n^\top V_n^{-1}z_n=x_n/(1+x_n)\le\log(1+x_n)$.
The determinant lemma identifies the sum of these logarithms as
$\log(\det V_N/\det\Lambda)$. Combining the inequalities proves
\cref{eq:constrained-comparison}. The argument is deterministic and requires
no covariance assumption.
\end{proof}

Let $\lambda=\lambda_{\min}(\Lambda)$. The trace of
$\Lambda^{-1/2}(V_N-\Lambda)\Lambda^{-1/2}$ is at most
$N/\lambda$. Applying concavity of $\log$ to its $p$ nonnegative eigenvalues
shows
\begin{equation}
\log\frac{\det V_N}{\det\Lambda}
\le p\log(1+N/(p\lambda)).
\label{eq:logdet-trace}
\end{equation}
For the identity product domain, apply \cref{lem:constrained-regret} to each
identity's subsequence, with dimension $p_I=d_\ell(q+1)$ and prior
$\lambda_I\mathrm{Id}_{p_I}$. Let $N_j$ count the observations for identity $j$ among
the first $N$ queries. Summing gives
$K\lambda_I+8p_I\sum_j\log(1+N_j/(p_I\lambda_I))$.
Including unqueried identities in the prior term only enlarges the bound:
their observed loss difference is zero.

\subsection{Pooling and conditional excess risk}
Pooling adds a fixed comparison term per identity, even when the shared
proposal changes or is misspecified. The comparison uses exponential weights
and exp-concavity \citep{cesabianchi2006prediction,hazan2023online}. We verify
the normalization and learning rate needed for this argument.
\begin{restatable}[Pathwise pool comparison]{lemma}{poolcomparison}
\label{lem:pool}
Fix an identity $j$ and let $N_j$ be its number of queries. On its $n$th
query, denote the normalized target by $y_n$, the answer feature by
$\phi_n$, the proposal coefficients by $a_n^H,a_n^P$, and their pooled
coefficient by $a_n$. These are the quantities from \cref{eq:pool} restricted
to this identity's queried subsequence: $|y_n|\le1$,
$\phi_n\in\R^{q+1}$ has norm at most one, and
$a_n^H,a_n^P\in\mathbb B_q$. For $0<\eta\le1/2$ and
$u\in\{H,P\}$,
\begin{equation}
\sum_{n=1}^{N_j}(y_n-a_n^\top\phi_n)^2
\le\sum_{n=1}^{N_j}(y_n-(a_n^u)^\top\phi_n)^2
 +\frac{4\log2}{\eta}.
\label{eq:pool-bound}
\end{equation}
The proposals may vary with the observed history.
\end{restatable}
\begin{proof}
For fixed $y\in[-1,1]$ and a scalar prediction $p\in[-1,1]$, put
$l(p)=(y-p)^2/4$.
Differentiating twice gives
\[
\frac{d^2}{dp^2}e^{-\eta l(p)}
=e^{-\eta l(p)}\{\eta^2(p-y)^2/4-\eta/2\}\le0
\]
when $0<\eta\le1/2$, since $(p-y)^2\le4$.
For round $n$, set $l_n(p)=(y_n-p)^2/4$. Write
$p_n^u=(a_n^u)^\top\phi_n$ and let $w_n^u$ be the weights on this
subsequence. For $p_n=a_n^\top\phi_n=\sum_u w_n^u p_n^u$, concavity implies
$e^{-\eta l_n(p_n)}\ge\sum_u w_n^u e^{-\eta l_n(p_n^u)}$.
Write the exponential-weights normalizer after $n$ queries as
$\mathcal Z_n=\frac12\sum_u\exp(-\eta\sum_{i\le n}l_i(p_i^u))$.
Then $\mathcal Z_0=1$ and
$\sum_u w_n^u e^{-\eta l_n(p_n^u)}=\mathcal Z_n/\mathcal Z_{n-1}$.
Taking logarithms and summing gives
$\sum_n l_n(p_n)\le-\log\mathcal Z_{N_j}/\eta$.
The contribution of either candidate to $\mathcal Z_{N_j}$ is at least
$\frac12\exp(-\eta\sum_n l_n(p_n^u))$; multiplying by four proves the claim.
The proposals may change between queries: the calculation compares their
realized losses on the same observations.
\end{proof}

We now compare the deployed response with the conditional optimum.
Combine \cref{lem:constrained-regret,lem:pool}, using candidate $H$ and
the fixed matrix $\Theta_j^*$ for each identity. For a queried round, define
\[
y_t=e_{t,I_t}/G,\quad p_t=a_{t,I_t}^\top\phi_q(d_{t,I_t}),\quad
p_t^*=\ell_{t,I_t}^\top\Theta_{I_t}^*\phi_q(d_{t,I_t}),
\]
and let $X_t=(y_t-p_t)^2-(y_t-p_t^*)^2$. The coefficients in $p_t$ and
$p_t^*$ are fixed under $\cF_t$; their evaluations depend on the purchased
answer. Set $X_t=0$ on free rounds.
For every prefix, pathwise,
\begin{equation}
\sum_{s\le t}X_s\le A_t,
\label{eq:pathwise-response}
\end{equation}
where $A_t$ is defined in \cref{eq:identity-pathwise-budget}.

\begin{restatable}[Conditional response-risk certificate]{lemma}{responseconversion}
\label{lem:response-conversion}
Under \cref{ass:response-model}, let
$\rho_t=\E[X_t\mid\cF_t]$, which is zero on free rounds.
For each $\zeta\in(0,1)$, with probability at least $1-\zeta$,
simultaneously for all prefixes,
\begin{equation}
\sum_{s\le t}\rho_s
\le2A_t+64\log(1/\zeta).
\label{eq:response-risk-envelope}
\end{equation}
\end{restatable}
\begin{proof}
The pathwise comparison alone concerns observed losses. To obtain a
conditional-risk bound, we show that the loss difference has a second
moment controlled by its conditional mean, then use a supermartingale.

\medskip\noindent\textbf{Second-moment control.} On a query, write $v_t=p_t-p_t^*$. Conditional
optimality on the convex coefficient ball implies its first-order
variational inequality
$\E[(p_t^*-y_t)v_t\mid\cF_t]\ge0$.
This follows by taking the directional derivative toward the feasible,
$\cF_t$-measurable coefficient $a_{t,I_t}$; boundedness justifies exchanging
this derivative with conditional expectation. Hence
\begin{equation}
\rho_t=\E[v_t^2+2(p_t^*-y_t)v_t\mid\cF_t]
\ge\E[v_t^2\mid\cF_t]\ge0.
\label{eq:bernstein-response}
\end{equation}
Because $|y_t|,|p_t|,|p_t^*|\le1$, we have $|X_t|\le4$ and
$X_t^2=v_t^2(p_t+p_t^*-2y_t)^2\le16v_t^2$. Thus
$\E[X_t^2\mid\cF_t]\le16\rho_t$ and $0\le\rho_t\le4$.
\medskip\noindent\textbf{Uniform concentration.}
For $\lambda=1/32$, use
$e^u\le1+u+\tfrac12u^2e^{|u|}$ to obtain
\begin{align*}
\E[e^{-\lambda X_t}\mid\cF_t]
&\le1-\lambda\rho_t+8\lambda^2e^{1/8}\rho_t\\
&\le1-\lambda\rho_t/2
\le e^{-\rho_t/64},
\end{align*}
where $8\lambda e^{1/8}<1/2$ verifies the middle inequality.
Define
\[
\mathcal M_t=
\exp\left\{\sum_{s\le t}\rho_s/64-\sum_{s\le t}X_s/32\right\},
\qquad \mathcal M_0=1.
\]
The preceding bound gives
$\E[\mathcal M_t\mid\cF_t]\le\mathcal M_{t-1}$.
Thus $(\mathcal M_t)$ is a nonnegative supermartingale with respect to
$(\cF_{t+1})$. Ville's maximal inequality
for nonnegative supermartingales \citep{howard2020timeuniform} bounds its probability
of ever exceeding $1/\zeta$ by $\zeta$. On the complementary event,
$\sum_{s\le t}\rho_s\le2\sum_{s\le t}X_s+64\log(1/\zeta)$ for every
prefix. Apply \cref{eq:pathwise-response}.
\end{proof}

The conversion relies on the variational inequality for the constrained
optimum. It does not require the oracle response to equal a conditional mean,
and it covers zero answer variance, boundary optima and nonunique optima.

\subsection{Proof of the response-error bound}
\label{app:response-certificate-proof}
\responsecertificate*
\begin{proof}
For the identity fits, \cref{lem:constrained-regret,lem:pool} give the
pathwise budget $A_t$ in \cref{eq:identity-pathwise-budget}. Under
\cref{ass:response-model}, the fixed comparator attains $W_{t,I_t}^*$
on every queried round. Thus the conditional mean of its normalized loss
difference $X_t$ is exactly the gap $\rho_t$ defined in \cref{sec:method}.
Apply \cref{lem:response-conversion} with
$\zeta=\delta/3$. With probability at least $1-\delta/3$, simultaneously
for all $t$, this gives
$\sum_{s\le t}\rho_s\le2A_t+64\log(3/\delta)=E_t$.
The equality is the definition in \cref{eq:identity-envelope}.
Finally, concavity of $\log$ and $\sum_jN_{t,j}\le T$ give
$A_t\le K\lambda_I+8Kp_I\log(1+T/(Kp_I\lambda_I))+4K\log2/\eta$.
Substitution yields $E_t\le\overline E_T$ for every $t\le T$.
\end{proof}

\subsection{Risk confidence with response-learning bias}
\label{app:risk-confidence}
We next bound the errors in the two risk estimates. The free labels have
centered noise; the queried labels also contain response-learning bias.
\begin{restatable}[Simultaneous risk confidence]{lemma}{riskconfidence}
\label{lem:risk-confidence}
Under the hypotheses and calibration of \cref{thm:main-regret}, with
probability at least $1-\delta$, the response bound
$\sum_{s\le t}\rho_s\le E_t$ and the free-risk bound below hold
simultaneously for every $t\le T$. The queried-risk bound also holds
for every $j\in\cA_t$ on each such round:
\begin{equation}
|\widehat L_t^0-L_t^0|
\le R^2B_0\|x_t^0\|_{(V_t^0)^{-1}}.
\label{eq:free-confidence}
\end{equation}
\begin{equation}
|\widetilde W_{t,j}-W_{t,j}^*|
\le8G^2B_*u_{t,j}
\label{eq:queried-confidence}
\end{equation}
\end{restatable}
\begin{proof}
We allocate failure probability $\delta/3$ to each of the response bound,
free-risk noise, and queried-risk noise. The weights are fixed before the
answer arrives, which preserves the conditional noise bounds.

\medskip\noindent\textbf{Regression noise.}
For the noise terms, we use the following standard consequence of
self-normalization \citep{abbasi2011improved}. Let $x_s\in\R^{d_V}$ be $\cF_s$-measurable,
let $\Lambda\in\R^{d_V\times d_V}$ be fixed and positive definite, and let
$\epsilon_s$ be $\cF_{s+1}$-measurable noise satisfying, for some $\sigma\ge0$,
\[
\E[\epsilon_s\mid\cF_s]=0,
\qquad
\E[e^{\gamma\epsilon_s}\mid\cF_s]\le e^{\gamma^2\sigma^2/2}
\quad\text{for every }\gamma\in\R.
\]
For $\zeta\in(0,1)$, with probability at least $1-\zeta$ simultaneously
for all prefixes,
\begin{equation}
\left\|\sum_{s<t}x_s\epsilon_s\right\|_{V_t^{-1}}
\le\sigma\sqrt{\log\frac{\det V_t}{\det\Lambda}+2\log(1/\zeta)},
\quad V_t=\Lambda+\sum_{s<t}x_sx_s^\top.
\label{eq:self-normalization}
\end{equation}
The result extends to the fixed positive-definite prior by whitening $x_s$ by
$\Lambda^{-1/2}$. The assumptions concern conditional moment-generating
functions and predictability; neither independent covariates nor independent
rounds are needed. For a variable in an interval of length one, its centered
version is conditionally $1/2$-sub-Gaussian by Hoeffding's lemma
\citep{hoeffding1963probability}.

\medskip\noindent\textbf{Free-risk confidence.}
Apply \cref{eq:self-normalization} with failure probability $\delta/3$ to the
centered free-risk labels. The uncentered labels lie in $[0,1]$, and their conditional
means are $(x_s^0)^\top\theta_0^*$ by \cref{ass:risk-model}.
The ridge identity, nonexpansiveness of clipping onto $[0,1]$, and
$\|\Lambda_0\theta_0^*\|_{(V_t^0)^{-1}}\le
\|\theta_0^*\|_{\Lambda_0}$ imply simultaneously
\cref{eq:free-confidence}.
Here \cref{eq:logdet-trace} bounds the determinant ratio by the expression in
\cref{eq:critic-calibration} for all $t\le T$.

\medskip\noindent\textbf{Queried-risk confidence.}
For the queried labels, put zero covariates and zero noise on free rounds.
On a query, write $c_s=c_{s,I_s}$. This feature is predictable, and
\cref{eq:upward-bias} gives
\[
Z_s=c_s^\top\theta_*+\rho_s/4+\epsilon_s,
\qquad \E[\epsilon_s\mid\cF_s]=0,
\]
where $\epsilon_s$ is conditionally $1/2$-sub-Gaussian.
Apply \cref{eq:self-normalization} again, at failure probability
$\delta/3$, to $x_s'=\sqrt{\omega_s}c_s$ and
$\epsilon_s'=\sqrt{\omega_s}\epsilon_s$ on queries and zero on free rounds.
Both $\omega_s$ and $c_s$ are $\cF_s$-measurable. Since
$0<\omega_s\le1$, the transformed noise remains conditionally
$1/2$-sub-Gaussian. The resulting martingale sum is
$\sum_s\omega_sc_s\epsilon_s$, and the determinant ratio is at most
$(1+T/(d_c\lambda_*))^{d_c}$. Including the prior norm bounds the
martingale and regularization terms by $B_*$ in the $(V_t^*)^{-1}$ norm,
simultaneously over $t\le T$.

The remaining term is the accumulated response-learning bias,
$b_t=\sum_{s<t:I_s\ne0}\omega_sc_s\rho_s/4$.
Indeed, the queried ridge error decomposes as
\[
\widehat\theta_t^*-\theta_*
=(V_t^*)^{-1}\left\{
\sum_{s<t:I_s\ne0}\omega_sc_s\epsilon_s+b_t-\Lambda_*\theta_*\right\}.
\]

On the response event of \cref{lem:response-conversion},
$\sum_{s<t}\rho_s\le\overline E_T$. The triangle inequality,
$\rho_s\ge0$, and $V_t^*\succeq V_s^*$ for $s<t$ give
\begin{equation}
\begin{split}
\|b_t\|_{(V_t^*)^{-1}}
&\le\tfrac14\sum_{s<t:I_s\ne0}\omega_s\rho_s\|c_s\|_{(V_t^*)^{-1}}\\
&\le\tfrac14\sum_{s<t:I_s\ne0}\omega_s\rho_su_{s,I_s}
 \le\tfrac14\tau\overline E_T=B_*.
\end{split}
\label{eq:main-bias-contraction}
\end{equation}
The last inequality uses $\omega_su_{s,I_s}\le\tau$, including zero width.
When there are no past queries, $b_t=0$. Cauchy--Schwarz and multiplication
by $4G^2$ now give
\cref{eq:queried-confidence}
for every eligible expert and every $t\le T$ on the common event.
The two risk-confidence events and the response conversion event have total
failure probability at most $\delta$. A vector-norm confidence event controls
every current feature, so no further union bound over experts is needed.
The noise inequality is established before intersecting with the response
event; its conditional moment-generating-function assumption is not asserted
after conditioning on that event.
\end{proof}

\subsection{Proof of the routing theorem}
\routingregret*
\explicitroutingbound*

\begin{proof}[Proof of \cref{thm:main-regret,prop:explicit-routing-bound}]
By \cref{lem:risk-confidence}, the response and risk-estimation bounds
hold on a common event with probability at least $1-\delta$.
We work on that event and first compare costs on one round, then sum the
result using bounds on regression widths.

\medskip\noindent\textbf{One-round comparison.} For a queried
round write $u_t=u_{t,I_t}$, and set $u_t=0$ on free rounds. Define a lower
bound on optimal queried risk and the free-risk estimation error by
\[
\underline W_{t,j}=\widetilde W_{t,j}-8G^2B_*u_{t,j},\qquad
\epsilon_t^0=\widehat L_t^0-L_t^0.
\]
With $\alpha=2B_*$, \cref{eq:queried-confidence} gives $\underline W_{t,j}\le W_{t,j}^*$.
The score is $S_{t,j}=\widehat L_t^0-\underline W_{t,j}-\beta_{t,j}$.
If an oracle action $j^*\ne0$ queries, then
$S_{t,j^*}\ge L_t^0-C_t^*+\epsilon_t^0$.
For the selected queried action $j$, its actual conditional cost is
$W_{t,j}^*+G^2\rho_t+\beta_{t,j}$ and
\begin{equation}
W_{t,j}^*-\underline W_{t,j}
\le16G^2B_*u_{t,j}.
\label{eq:selected-slack}
\end{equation}
We compare the selected action with an oracle action, separating the four
possible pairs of free and queried choices. If both actions query, score
maximality gives
$\underline W_{t,j}+\beta_{t,j}\le\underline W_{t,j^*}+\beta_{t,j^*}\le C_t^*$, and hence
\[
\E[C_t\mid\cF_t]-C_t^*
\le W_{t,j}^*-\underline W_{t,j}+G^2\rho_t;
\]
the free-risk error cancels. If the selected action queries and the oracle
is free, $S_{t,j}\ge0$ gives the same bound plus $\epsilon_t^0$.
If the selected action is free and the oracle queries,
$0\ge S_{t,j^*}\ge L_t^0-C_t^*+\epsilon_t^0$, so the regret is at most
$-\epsilon_t^0$. If both are free the regret is zero. Let $\one\{\cdot\}$ denote the indicator of an event. In all cases,
\begin{equation}
\E[C_t\mid\cF_t]-C_t^*
\le |\epsilon_t^0|+\one\{I_t\ne0\}G^2
\{16B_*u_t+\rho_t\}.
\label{eq:one-step-regret}
\end{equation}
This derivation also handles an empty eligible set.

\medskip\noindent\textbf{Summing regression widths.}
To sum the uncertainty terms, we bound cumulative regression widths by the
number of observations. Let $x_i\in\R^{d_V}$ be features of norm at most one,
where $d_V$ is a positive integer. Let $\Lambda\succeq\lambda\mathrm{Id}_{d_V}$
with $\lambda>0$ and $V_n=\Lambda+\sum_{i<n}x_ix_i^\top$.
Then $v_n=x_n^\top V_n^{-1}x_n\le\lambda^{-1}$ and
$v_n\le(1+\lambda^{-1})\log(1+v_n)$, using
$\log(1+v)\ge v/(1+v)$ for $v\ge0$.
The determinant lemma and \cref{eq:logdet-trace} give
\[
\sum_{n=1}^Nv_n\le(1+\lambda^{-1})d_V\log(1+N/(d_V\lambda)).
\]
Cauchy--Schwarz therefore gives
$\sum_{n=1}^N\sqrt{v_n}\le\mathcal W(N,d_V,\lambda)$, also valid for $N=0$.
Apply this bound to the free-risk features on all $T$ rounds. The free-risk regression
receives an outcome on every round, including
rounds on which an expert is queried.
For the queried-risk design, the determinant lemma and
$\omega_su_s^2\le\lambda_*^{-1}$ give
\[
\sum_{s:I_s\ne0}\omega_su_s^2
\le(1+\lambda_*^{-1})\log\frac{\det V_{T+1}^*}{\det\Lambda_*}
\le\mathcal H(N_T).
\]
The whitened update matrix has trace at most $N_T/\lambda_*$, which gives
the last inequality. On rounds with $\omega_s=1$, Cauchy--Schwarz bounds
the sum of widths by $\sqrt{N_T\mathcal H(N_T)}$. On the remaining
queried rounds, $\omega_su_s^2=\tau u_s$, so
\[
\sum_{s:I_s\ne0}u_s
\le\sqrt{N_T\mathcal H(N_T)}+\mathcal H(N_T)/\tau.
\]
This also holds when $N_T=0$. Sum \cref{eq:one-step-regret}, use
\cref{eq:free-confidence}, $\sum_t\rho_t\le E_T$, and substitute
$\tau=4B_*/\overline E_T$.
This proves \cref{eq:explicit-main-bound} and hence
\cref{prop:explicit-routing-bound}.

\medskip\noindent\textbf{Horizon dependence.}
For $\alpha=2B_*$ and fixed problem parameters,
$E_T\le C_1\log((T+1)/\delta)$ and
$B_v\le C_2\sqrt{\log((T+1)/\delta)}$. The width envelope is at most
$C_3\sqrt{T\log(T+1)}$, and the additive response term is at most
$C_4\log((T+1)/\delta)\log(T+1)$ for a fixed finite $C_4$.
Using $\log(T+1)\le\sqrt T$ gives
\cref{eq:main-rate}, with constants depending only on the quantities listed
in the theorem. This proves \cref{thm:main-regret}.
\end{proof}

\input{appendix/d_approximate_models}

%% file: appendix/d_approximate_models.tex
\subsection{Approximate response and risk models}
\label{app:approximate-models}
The exact models in \cref{ass:response-model,ass:risk-model} connect
observations across rounds. Here fixed models may approximate the optimal
responses and risks; the comparator remains the same conditional oracle.
The response proof must account for the fixed comparator's approximation
error. The risk proof must also handle signed model errors, separately
from the nonnegative bias caused by response learning.

\medskip\noindent\textbf{Approximation assumptions.}
Fix $T\ge1$, $\delta\in(0,1)$ and deterministic tolerances
$\varepsilon_R,\varepsilon_0,\varepsilon_*\ge0$ before interaction.
For each identity, let $\overline\Theta_j\in\R^{d_\ell\times(q+1)}$
be a fixed matrix with $\|\overline\Theta_j\|_F\le1$, and write
$\overline a_{t,j}=\overline\Theta_j^\top\ell_{t,j}$.
Assume, almost surely for every $t\le T$ and $j\in\cA_t$,
\begin{equation}
0\le\epsilon_{t,j}:=
\E\!\left[\left(e_{t,j}/G-\overline a_{t,j}^\top
\phi_q(d_{t,j})\right)^2\mid\cF_t\right]
-W_{t,j}^*/G^2\le\varepsilon_R.
\label{eq:approximate-response-model}
\end{equation}
Nonnegativity follows from feasibility. The error measures excess risk
over the conditional ball optimum, rather than distance between coefficient
vectors. This distinction allows boundary optima and singular answer laws.
Also assume fixed parameters $\overline\theta_0\in\R^{d_0}$ and
$\overline\theta_*\in\R^{d_c}$ satisfy
$\|\overline\theta_0\|_{\Lambda_0}\le Q_0$ and
$\|\overline\theta_*\|_{\Lambda_*}\le Q_*$, with
\begin{equation}
\left|L_t^0/R^2-(x_t^0)^\top\overline\theta_0\right|\le\varepsilon_0,
\qquad
\left|W_{t,j}^*/(4G^2)-c_{t,j}^\top\overline\theta_*\right|
\le\varepsilon_*.
\label{eq:approximate-risk-models}
\end{equation}
The free-risk inequality holds almost surely on every round, including
when $\cA_t$ is empty; the queried-risk inequality holds almost surely
for every eligible expert on every round.
These parameters and matrices are fixed before the future observations
are realized. Their values need not be known to the learner. Calibration
uses the supplied bounds $\varepsilon_R,\varepsilon_*$; the free-risk
tolerance $\varepsilon_0$ enters the guarantee but not the calibration.

Let $A_t$ be the pathwise response budget in
\cref{eq:identity-pathwise-budget}, and define its deterministic envelope
and the cumulative response approximation error by
\begin{equation}
\begin{split}
\overline A_T&=K\lambda_I+8Kp_I
\log\!\left(1+\frac{T}{Kp_I\lambda_I}\right)
+\frac{4K\log2}{\eta},\\
\mathcal D_t&=\sum_{s\le t:I_s\ne0}\epsilon_{s,I_s},
\qquad p_I=d_\ell(q+1).
\end{split}
\label{eq:approximate-pathwise-budgets}
\end{equation}
The response certificate and its horizon envelope are
\begin{equation}
E_t^{\rm app}=2A_t+3\mathcal D_t+128\log(3/\delta),\qquad
\overline E_T^{\rm app}=2\overline A_T+3T\varepsilon_R
+128\log(3/\delta).
\label{eq:approximate-envelopes}
\end{equation}
Unlike $A_t$, the quantity $\mathcal D_t$ is unobserved. The algorithm uses
the deterministic envelope $\overline E_T^{\rm app}$, which is at least
$E_t^{\rm app}$ for every $t\le T$. With $B_v$ as in
\cref{eq:main-noise-radius}, set
\begin{equation}
B_*^{\rm app}=B_*+\varepsilon_*\sqrt T,\qquad
\alpha=2B_*^{\rm app},\qquad
\tau=\frac{4B_*^{\rm app}}{\overline E_T^{\rm app}}.
\label{eq:approximate-calibration}
\end{equation}

\begin{restatable}[Routing under approximate models]{proposition}{approximaterouting}
\label{prop:approximate-models}
Assume the bounded non-interventional protocol of \cref{sec:problem} and
the approximation assumptions above, including the fixed feasible matrices,
parameter-norm bounds and \cref{eq:approximate-response-model,eq:approximate-risk-models}.
Choose feature maps and positive-definite priors before interaction, with
public feature vectors of norm at most one. Run the
\ORUCB{} updates with exact constrained response fits, full-history risk
regressions, $0<\eta\le1/2$, and \cref{eq:approximate-calibration}.
With probability at least $1-\delta$ over the joint interaction law,
\begin{equation}
\begin{split}
\Reg_T\le{}&R^2\left[(B_0+\varepsilon_0\sqrt T)
\mathcal W(T,d_0,\lambda_0)+T\varepsilon_0\right]\\
&+G^2\left[16B_*^{\rm app}\sqrt{N_T\mathcal H(N_T)}
+4\overline E_T^{\rm app}\mathcal H(N_T)
+E_T^{\rm app}+8T\varepsilon_*\right],
\end{split}
\label{eq:approximate-regret-bound}
\end{equation}
with $\mathcal W,\mathcal H$ defined in \cref{eq:width-envelope,eq:main-widths} and
$\Reg_T$ measured against the original oracle in \cref{eq:regret}.
\end{restatable}

\begin{proof}
We first adapt the conditional response certificate, then bound the two
risk estimates and repeat the one-round routing comparison.

\medskip\noindent\textbf{Response approximation.}
On a queried round, abbreviate $y_t=e_{t,I_t}/G$,
$p_t=a_{t,I_t}^\top\phi_q(d_{t,I_t})$ and
$\overline p_t=\overline a_{t,I_t}^\top\phi_q(d_{t,I_t})$.
Let $p_t^*=(a_{t,I_t}^*)^\top\phi_q(d_{t,I_t})$ for a measurable
conditional minimizer over $\mathbb B_q$. The three coefficients are
$\cF_t$-measurable, and the predictions and target belong to $[-1,1]$.
Define
\[
X_t=(y_t-p_t)^2-(y_t-\overline p_t)^2,\qquad
\epsilon_t=\epsilon_{t,I_t},\qquad
\rho_t=\E[(y_t-p_t)^2\mid\cF_t]-W_{t,I_t}^*/G^2,
\]
and set all three quantities to zero on free rounds.
Both $\rho_t$ and $\epsilon_t$ are nonnegative, while
$\E[X_t\mid\cF_t]=\rho_t-\epsilon_t$ can have either sign.
The pathwise regression and pool bounds compare to any fixed feasible
matrices, so \cref{lem:constrained-regret,lem:pool} still give
$\sum_{s\le t}X_s\le A_t$ for every prefix.

Conditional optimality of $a_{t,I_t}^*$ gives the variational inequality
\[
\E[(p_t^*-y_t)(p_t-p_t^*)\mid\cF_t]\ge0,
\]
and the same inequality with $p_t$ replaced by $\overline p_t$.
Expanding the two excess risks therefore gives
\[
\E[(p_t-p_t^*)^2\mid\cF_t]\le\rho_t,
\qquad
\E[(\overline p_t-p_t^*)^2\mid\cF_t]\le\epsilon_t.
\]
Since $|X_t|\le4$ and
$X_t^2=(p_t-\overline p_t)^2(p_t+\overline p_t-2y_t)^2$,
\begin{equation}
\E[X_t^2\mid\cF_t]
\le16\E[(p_t-\overline p_t)^2\mid\cF_t]
\le32(\rho_t+\epsilon_t).
\label{eq:approximate-second-moment}
\end{equation}
For $\lambda=1/64$, the exponential Taylor bound used in
\cref{lem:response-conversion} now gives
\begin{align*}
\E[e^{-\lambda X_t}\mid\cF_t]
&\le1-\lambda(\rho_t-\epsilon_t)
 +16\lambda^2e^{4\lambda}(\rho_t+\epsilon_t)\\
&\le1-\lambda\rho_t/2+3\lambda\epsilon_t/2
\le\exp(-\rho_t/128+3\epsilon_t/128).
\end{align*}
The middle inequality uses $16\lambda e^{4\lambda}<1/2$.
Thus
\[
\exp\!\left\{\sum_{s\le t}\rho_s/128
-3\sum_{s\le t}\epsilon_s/128-\sum_{s\le t}X_s/64\right\}
\]
is a nonnegative supermartingale starting at one, with respect to
$(\cF_{t+1})$. Ville's inequality, at failure probability $\delta/3$,
implies simultaneously for every prefix $t\le T$
\begin{equation}
\sum_{s\le t}\rho_s
\le2\sum_{s\le t}X_s+3\mathcal D_t+128\log(3/\delta)
\le E_t^{\rm app}.
\label{eq:approximate-response-certificate}
\end{equation}

\medskip\noindent\textbf{Risk approximation.}
Write the two signed model residuals as
$r_t^0=L_t^0/R^2-(x_t^0)^\top\overline\theta_0$ and
$r_{t,j}^*=W_{t,j}^*/(4G^2)-c_{t,j}^\top\overline\theta_*$.
On a query, the normalized label satisfies
\[
Z_t=c_{t,I_t}^\top\overline\theta_*+r_{t,I_t}^*
+\rho_t/4+\epsilon_t^{\rm noise},\qquad \E[\epsilon_t^{\rm noise}\mid\cF_t]=0.
\]
Because $Z_t\in[0,1]$, the centered noise is conditionally
$1/2$-sub-Gaussian. Applying \cref{eq:self-normalization} to the
predictably weighted features bounds the noise and regularization terms
by $B_*$ in the $(V_t^*)^{-1}$ norm, with failure probability $\delta/3$.

The signed model errors admit a separate bound. Let $U$ have columns
$\sqrt{\omega_s}c_{s,I_s}$ over past queries, and let $v$ have entries
$\sqrt{\omega_s}r_{s,I_s}^*$. Since $V_t^*=\Lambda_*+UU^\top$,
$\|(V_t^*)^{-1/2}U\|_{\rm op}\le1$, and hence
\begin{equation}
\left\|\sum_{s<t:I_s\ne0}\omega_sc_{s,I_s}r_{s,I_s}^*
\right\|_{(V_t^*)^{-1}}
\le\|v\|_2
\le\varepsilon_*\sqrt T.
\label{eq:approximate-risk-contraction}
\end{equation}
Here $\|\cdot\|_{\rm op}$ denotes the Euclidean operator norm; its bound
follows from $UU^\top\preceq V_t^*$. On the response event,
design monotonicity and $\omega_su_{s,I_s}\le\tau$ also give
\[
\left\|\tfrac14\sum_{s<t:I_s\ne0}\omega_sc_{s,I_s}\rho_s
\right\|_{(V_t^*)^{-1}}
\le\tfrac\tau4\sum_{s<t}\rho_s
\le\tfrac\tau4\overline E_T^{\rm app}=B_*^{\rm app}.
\]
The noise, regularization and signed model-error terms total at most
$B_*^{\rm app}$, so evaluating the ridge error at a current feature and
including its own residual gives
\begin{equation}
|\widetilde W_{t,j}-W_{t,j}^*|
\le8G^2B_*^{\rm app}u_{t,j}+4G^2\varepsilon_*
\quad(t\le T,\ j\in\cA_t).
\label{eq:approximate-queried-confidence}
\end{equation}
The vector-norm event controls all eligible experts simultaneously.

For the free regression, the same matrix argument with columns $x_s^0$
and residual vector $(r_s^0)_{s<t}$ gives
\[
\left\|\sum_{s<t}x_s^0r_s^0\right\|_{(V_t^0)^{-1}}
\le\sqrt{\sum_{s<t}(r_s^0)^2}\le\varepsilon_0\sqrt T.
\]
Self-normalization at failure probability $\delta/3$, the prior-norm
bound and clipping toward the true normalized risk in $[0,1]$ imply
\begin{equation}
|\widehat L_t^0-L_t^0|
\le R^2\left[(B_0+\varepsilon_0\sqrt T)
\|x_t^0\|_{(V_t^0)^{-1}}+\varepsilon_0\right].
\label{eq:approximate-free-confidence}
\end{equation}
The response event and the two noise events hold together with probability
at least $1-\delta$. The conditional noise inequalities are established
before this intersection, as in the exact-model proof.

\medskip\noindent\textbf{Routing comparison.}
Work on that common event and put
$\underline W_{t,j}=\widetilde W_{t,j}-8G^2B_*^{\rm app}u_{t,j}$.
Equation~\eqref{eq:approximate-queried-confidence} gives
$\underline W_{t,j}\le W_{t,j}^*+4G^2\varepsilon_*$ and
$W_{t,j}^*-\underline W_{t,j}\le16G^2B_*^{\rm app}u_{t,j}+4G^2\varepsilon_*$.
If both learner and oracle query, score maximality compares their lower
risk estimates plus fees. Each expert contributes at most
$4G^2\varepsilon_*$ of model error, giving $8G^2\varepsilon_*$ in total.
If only one queries, the comparison incurs at most
$4G^2\varepsilon_*$ and the absolute free-risk error;
if both are free, regret is zero. Therefore
\begin{equation}
\begin{split}
\E[C_t\mid\cF_t]-C_t^*
\le{}&|\widehat L_t^0-L_t^0|+8G^2\varepsilon_*\\
&+\one\{I_t\ne0\}G^2
\{16B_*^{\rm app}u_{t,I_t}+\rho_t\}.
\end{split}
\label{eq:approximate-one-step}
\end{equation}
The width summations in the exact-model proof hold for any $\tau>0$,
including \cref{eq:approximate-calibration}:
\[
\sum_{t:I_t\ne0}u_{t,I_t}
\le\sqrt{N_T\mathcal H(N_T)}+\mathcal H(N_T)/\tau,
\qquad
\sum_{t=1}^T\|x_t^0\|_{(V_t^0)^{-1}}
\le\mathcal W(T,d_0,\lambda_0).
\]
Summing \cref{eq:approximate-one-step}, using
\cref{eq:approximate-response-certificate,eq:approximate-free-confidence},
and substituting $\tau=4B_*^{\rm app}/\overline E_T^{\rm app}$ proves
\cref{eq:approximate-regret-bound}. Empty query sequences and empty
eligible sets are covered by the same bounds.
\end{proof}

\medskip\noindent\textbf{Dependence on approximation error.}
For fixed problem parameters and $\delta=(T+1)^{-2}$, the right-hand side
of \cref{eq:approximate-regret-bound} is at most a constant times
\[
\sqrt T\log(T+1)
+T\varepsilon_R\log(T+1)
+T(\varepsilon_0+\varepsilon_*)\sqrt{\log(T+1)}.
\]
The constant depends on the same fixed problem quantities as in
\cref{thm:main-regret}, and not on the tolerances or $T$.
At zero approximation error this recovers the original rate;
\cref{thm:main-regret} gives the sharper response-concentration constant.
For fixed positive errors, the bound's approximation terms are linear in
$T$ up to logarithmic factors. The conditional laws may vary subject to
the stated approximation bounds.
If a known deterministic bound $\overline{\mathcal D}_T$ on
$\mathcal D_T$ is available, it can replace $T\varepsilon_R$ in the
horizon envelope. Such a bound must hold almost surely for the analyzed
interaction; a bound holding only with high probability requires its
failure probability to be included.

%% file: appendix/h_experiment_details.tex
\section{Additional experiments}
\label{app:experiment-details}
\label{app:experimental-setup}
We report (i) learning curves on the four benchmark streams
(Appendix~\ref{app:benchmark-queries}); (ii) cost and query rates at six
consultation prices (Appendix~\ref{app:fee-sensitivity}); and (iii) routing
comparisons with a common response learner
(Appendix~\ref{app:jena-routing}).
Appendix~\ref{app:data-construction} describes the data and expert construction.

\textbf{Common protocol.}
\label{app:benchmark-protocol}
We rank methods by mean squared prediction error plus the consultation
fees actually paid. A free prediction incurs no fee. All methods receive
the same public information, availability schedule, and fee for each
expert; only the selected expert's answer is revealed.
Reported real-data costs are in squared-temperature units.
We select hyperparameters by fee-inclusive cost on validation data and
report results on separate test data. This protocol applies to the four
benchmarks, every consultation price, and the matched-response comparisons.
Hyperparameters remain fixed during test evaluation; the methods continue
their prescribed online updates after each round's feedback.
The benchmark search budget is $256$ configurations per method and dataset.

\textbf{Repetitions and reporting.}
Benchmark means and sample standard deviations use five learner seeds
on each fixed stream. Deterministic methods have identical replays;
we omit zero standard deviations.
Bold values indicate the lowest displayed cost within a comparison,
including ties. Exact configurations, seeds, and trajectories are retained
with the experimental records.

\subsection{Learning on the four benchmark streams}
\label{app:benchmark-queries}
\label{app:datasets}
\label{app:benchmark-implementation}
\textbf{In brief.}
\ORUCB{} has the lowest final mean total cost among the eight methods on
all four streams. The trajectories show how these endpoint comparisons
develop as the methods learn from selected answers.

\textbf{Data.}
We use a synthetic stream, daily Melbourne temperatures
\citep{daily_min_temp_melbourne}, four-hourly Jena observations
\citep{jena_climate}, and daily Delhi observations
\citep{daily_delhi_climate}. The synthetic process alternates between
regimes favoring free prediction, direct deferral, linear correction,
and nonlinear correction. The real-data splits are chronological:
Melbourne uses a $1200$-observation fitting prefix that is also replayed
for offline warmup; Jena uses $2000$ fitting and $2000$ warmup observations;
Delhi uses $600$ fitting and $250$ warmup observations.
Jena standardization uses its full fitting-and-warmup prefix.
No scored observation enters these preprocessing fits.

\textbf{Experts.}
Synthetic and Melbourne each have twelve simulated experts in four
families, with four experts available per round. Melbourne's private
sensors combine real temperatures with controlled bias and noise.
Jena has twelve frozen models from random forests, extra trees, gradient
boosting, and radial-basis kernel ridge, with four available per round.
Delhi has six frozen tree-based models, with four to six available.
An internal ridge predictor updates after each revealed real-data outcome.
Predetermined availability schedules create absence and return, allowing
the response learner to reuse an identity's earlier observations.
The seven external baselines deploy purchased answers unchanged; their
linear, sampling, neural, discounted, and restart-based implementations
follow the methods cited in \cref{sec:experiments}.

\textbf{Metric and fees.}
Every point in \cref{fig:learning-curves} is the running mean total cost
from the first scored round, not a moving-window average.
Per-query fees are $0.07$, approximately $0.30$, approximately $0.17$,
and $0.02$ on synthetic, Melbourne, Jena, and Delhi, respectively.
Lines average the five learner seeds; bands show one sample standard deviation.

\textbf{Results.}
The final costs are $0.08092$, $4.49069$, $3.72844$, and $2.59240$,
respectively, compared with $0.08316$, $4.61400$, $4.02217$, and
$2.63729$ for the best external baseline on each stream.
\Cref{tab:main-costs} gives the complete endpoint comparison.
The plots retain the full scored intervals and every baseline.

\begin{figure}[H]
\centering
\includegraphics[width=\linewidth]{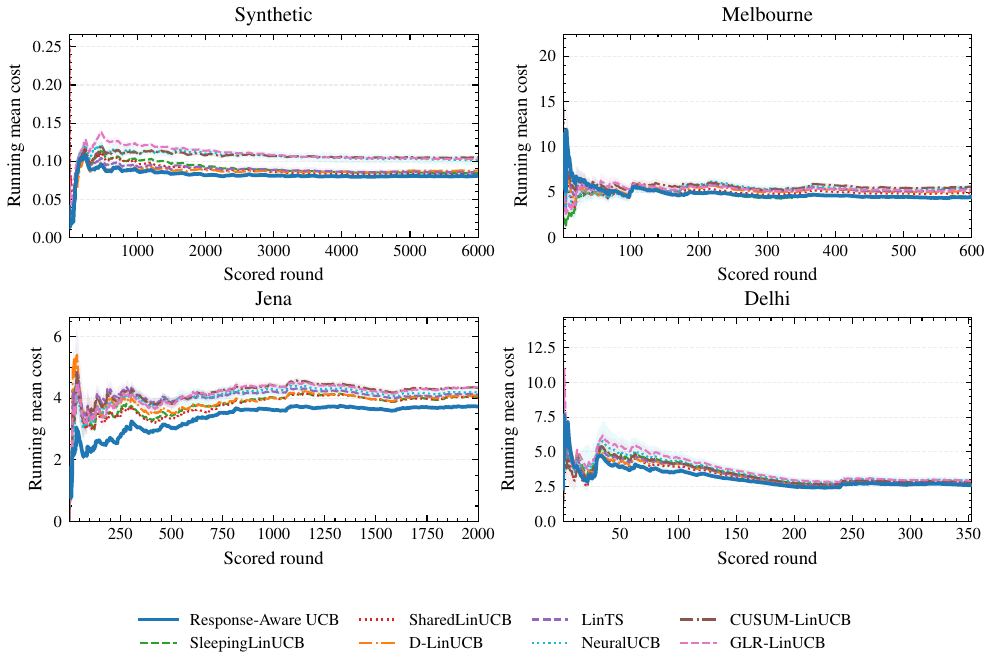}
\caption{Fee-inclusive test learning curves on all four streams.
Every scored round is included, with linear cost axes starting at zero.
Running means start at the first scored round. Shading shows one
learner-seed standard deviation where nonzero.}
\label{fig:full-learning-curves}
\label{fig:learning-curves}
\end{figure}

\subsection{Cost-sensitive consultation}
\label{app:fee-sensitivity}
\textbf{In brief.}
Across the six price-specific configurations, \ORUCB{} reduces its query
fraction from $96\%$ to $0.1\%$ while retaining lower mean total cost than
the external baselines under equal hyperparameter-search budgets.

\textbf{Data.}
We use the same Jena validation and test blocks as the benchmark.
Every price is evaluated by replaying warmup and scoring, so a new fee can
change both the purchased observations and the learned responses.

\textbf{Experts and comparisons.}
We retain the twelve frozen Jena experts, their four-family availability
schedule, and the same internal predictor.
We compare with all seven external baselines.
For each price, we select every method's configuration on the validation
block and evaluate it on the test block, using equal search budgets for
\ORUCB{} and all seven external baselines.
\label{app:price-selection}

\textbf{Metrics and fees.}
A multiplier $m$ changes the common per-query fee to
$\beta(m)\approx 0.17m$; the unrounded base fee is used in every replay.
The two metrics are mean total cost and the fraction of scored rounds
with a query. Their joint interpretation matters: avoiding a query saves
its fee but can increase prediction error.
The six multipliers $1,3,5,7.5,12.5,50$ span frequent consultation to
near-free prediction.
Reported standard deviations describe learner-seed variability.

\textbf{Results at six prices.}
\Cref{tab:price-cost} presents the external comparisons, with cost and
query percentage in separate panels.
\ORUCB{} has lower cost than every external baseline at each price.
Its query rates are $96\%$, $79.4\%$, $9.7\%$, $8.15\%$, $1.85\%$,
and $0.10\%$. The base-price configuration in this study is selected
separately from the benchmark configuration.

\begin{table}[H]
\centering\footnotesize
\setlength{\tabcolsep}{1.6pt}
\renewcommand{\arraystretch}{1.2}
\caption{Jena test performance against external baselines at six prices. Panels share the same methods and six
prices; fees and total costs are in squared-temperature units.
Entries show mean and nonzero sample standard deviation.
Bold marks minimum displayed total cost in each column; query percentages
are not ranked. The configuration-selection protocol is stated above.}
\label{tab:price-cost}
\label{tab:price-query}
\label{tab:main-price-range}
\textbf{(a) Mean total cost}\par\smallskip
\input{output/pdf/appendix_experiments/price_baselines_cost}
\par\medskip
\textbf{(b) Query percentage}\par\smallskip
\input{output/pdf/appendix_experiments/price_baselines_query}
\end{table}

\begin{figure}[H]
\centering
\includegraphics[width=\linewidth]{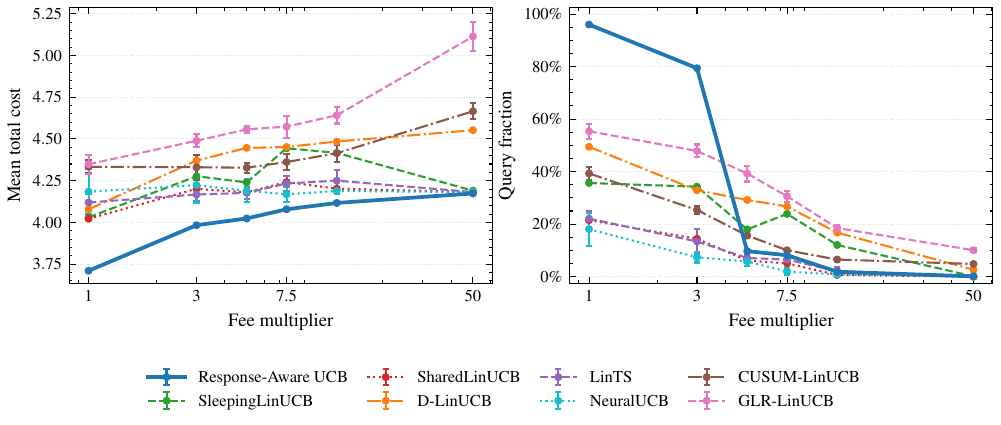}
\caption{Six-price test comparison with external baselines: total cost and
query fraction. The fee axis is logarithmic. Error bars show one
sample standard deviation when nonzero; lines join the measured prices
without smoothing. The cost axis has a nonzero origin.}
\label{fig:main-retuned-fees}
\end{figure}
\subsection{Routing with the same response learner}
\label{app:jena-routing}
\textbf{In brief.}
With the response architecture held fixed, \ORUCB{} attains the lowest
cost on all four streams, tying joint response UCB on three and improving
on both UCB alternatives on Delhi.

\textbf{Data.}
We use all four benchmark streams with the same fitting, warmup,
validation, and test intervals. Public information and preprocessing are unchanged
from Appendix~\ref{app:benchmark-queries}.

\textbf{Experts and routers.}
Every querying method uses the same experts and cubic correction learner,
including shared and identity responses, pooling, and retained memory.
Joint response UCB retains our shared and identity risk features but uses
ordinary unit-weight ridge regression. Disjoint response UCB instead fits
a separate ordinary risk regression for each expert.
Both alternatives select their confidence coefficient and risk
regularization on validation data using the same $256$-configuration budget per dataset;
response settings remain fixed.
Cyclic response queries the next available identity in registry order;
free prediction never queries. Each router learns from its own purchases,
so its fitted responses reflect its acquisition decisions.

\textbf{Metric and fees.}
We report mean squared error plus the same benchmark consultation fees:
$0.07$, approximately $0.30$, approximately $0.17$, and $0.02$.
The risk-regression target is the loss of the deployed, corrected prediction,
not the loss of the expert answer before correction.

\textbf{Results.}
\Cref{tab:jena-routing} compares the routers across all four streams.
On Delhi, \ORUCB{} costs $2.59240$, compared with
$2.62820$ for joint response UCB and $2.62465$ for disjoint response UCB.
Joint response UCB matches the benchmark costs on the other three streams.
Both cyclic querying and never querying have higher cost than \ORUCB{}
on every stream.

\begin{table}[H]
\centering\small
\setlength{\tabcolsep}{6pt}
\renewcommand{\arraystretch}{1.2}
\caption{Matched-response routing: mean total test cost on all four streams.
All querying methods can correct the selected expert's answer.
Bold values mark the minimum per dataset.}
\label{tab:jena-routing}
\label{tab:main-routing}
\input{output/pdf/appendix_experiments/routing}

\end{table}

%% file: output/pdf/appendix_experiments/price_baselines_cost.tex
\begin{tabular}{lcccccc}
\toprule
Fee multiplier & $\times 1$ & $\times 3$ & $\times 5$ & $\times 7.5$ & $\times 12.5$ & $\times 50$ \\
Fee per query & 0.17 & 0.51 & 0.85 & 1.28 & 2.13 & 8.52 \\
\midrule
\ORUCB{} & \textbf{3.711} & \textbf{3.984} & \textbf{4.024} & \textbf{4.079} & \textbf{4.117} & \textbf{4.174} \\
SleepingLinUCB & 4.030 & 4.277 & 4.240 & 4.444 & 4.417 & 4.190 \\
SharedLinUCB & 4.022 & 4.201 & 4.180 & 4.242 & 4.202 & 4.182 \\
D-LinUCB & 4.078 & 4.371 & 4.447 & 4.452 & 4.485 & 4.552 \\
LinTS & 4.119 {\scriptsize$\pm$ 0.049} & 4.167 {\scriptsize$\pm$ 0.040} & 4.177 {\scriptsize$\pm$ 0.036} & 4.233 {\scriptsize$\pm$ 0.044} & 4.250 {\scriptsize$\pm$ 0.065} & 4.182 \\
NeuralUCB & 4.185 {\scriptsize$\pm$ 0.134} & 4.223 {\scriptsize$\pm$ 0.108} & 4.190 {\scriptsize$\pm$ 0.069} & 4.169 {\scriptsize$\pm$ 0.045} & 4.189 {\scriptsize$\pm$ 0.020} & 4.182 \\
CUSUM-LinUCB & 4.333 {\scriptsize$\pm$ 0.039} & 4.331 {\scriptsize$\pm$ 0.073} & 4.328 {\scriptsize$\pm$ 0.030} & 4.362 {\scriptsize$\pm$ 0.050} & 4.415 {\scriptsize$\pm$ 0.051} & 4.666 {\scriptsize$\pm$ 0.048} \\
GLR-LinUCB & 4.349 {\scriptsize$\pm$ 0.056} & 4.489 {\scriptsize$\pm$ 0.038} & 4.557 {\scriptsize$\pm$ 0.021} & 4.574 {\scriptsize$\pm$ 0.066} & 4.641 {\scriptsize$\pm$ 0.049} & 5.113 {\scriptsize$\pm$ 0.085} \\
\bottomrule
\end{tabular}

%% file: output/pdf/appendix_experiments/price_baselines_query.tex
\begin{tabular}{lcccccc}
\toprule
Fee multiplier & $\times 1$ & $\times 3$ & $\times 5$ & $\times 7.5$ & $\times 12.5$ & $\times 50$ \\
Fee per query & 0.17 & 0.51 & 0.85 & 1.28 & 2.13 & 8.52 \\
\midrule
\ORUCB{} & 96.00 & 79.40 & 9.70 & 8.15 & 1.85 & 0.10 \\
SleepingLinUCB & 35.65 & 34.25 & 18.00 & 23.90 & 12.05 & 0.10 \\
SharedLinUCB & 21.50 & 14.40 & 6.10 & 5.05 & 0.65 & 0.00 \\
D-LinUCB & 49.45 & 32.90 & 29.20 & 26.75 & 16.75 & 2.65 \\
LinTS & 22.07 {\scriptsize$\pm$ 3.04} & 13.38 {\scriptsize$\pm$ 4.83} & 7.10 {\scriptsize$\pm$ 1.22} & 6.49 {\scriptsize$\pm$ 4.13} & 2.10 {\scriptsize$\pm$ 1.63} & 0.00 \\
NeuralUCB & 18.07 {\scriptsize$\pm$ 6.33} & 7.37 {\scriptsize$\pm$ 2.00} & 5.78 {\scriptsize$\pm$ 1.72} & 1.90 {\scriptsize$\pm$ 1.26} & 0.82 {\scriptsize$\pm$ 0.80} & 0.00 \\
CUSUM-LinUCB & 39.31 {\scriptsize$\pm$ 2.51} & 25.38 {\scriptsize$\pm$ 1.66} & 15.65 {\scriptsize$\pm$ 0.74} & 10.03 {\scriptsize$\pm$ 0.36} & 6.50 {\scriptsize$\pm$ 0.41} & 4.82 {\scriptsize$\pm$ 0.42} \\
GLR-LinUCB & 55.35 {\scriptsize$\pm$ 2.94} & 47.87 {\scriptsize$\pm$ 2.46} & 39.26 {\scriptsize$\pm$ 3.03} & 30.60 {\scriptsize$\pm$ 2.10} & 18.50 {\scriptsize$\pm$ 1.29} & 10.08 {\scriptsize$\pm$ 0.70} \\
\bottomrule
\end{tabular}

%% file: output/pdf/appendix_experiments/routing.tex
\begin{tabular}{lcccc}
\toprule
Router & Synthetic & Melbourne & Jena & Delhi \\
\midrule
\ORUCB{} & \textbf{0.08092} & \textbf{4.49069} & \textbf{3.72844} & \textbf{2.59240} \\
Joint response UCB & \textbf{0.08092} & \textbf{4.49069} & \textbf{3.72844} & 2.62820 \\
Disjoint response UCB & 0.08358 & 4.62449 & 3.74382 & 2.62465 \\
Cyclic response & 0.15557 & 5.66927 & 3.97404 & 2.84056 \\
Free prediction & 0.10949 & 5.40136 & 4.18186 & 2.72163 \\
\bottomrule
\end{tabular}

%% file: appendix/g_data_construction.tex
\section{Data and expert construction}
\label{app:data-construction}
The benchmarks use two types of expert answers. Synthetic and Melbourne
use simulated signals whose bias, noise, and nonlinear distortion we can
control. Jena and Delhi use predictions from models fitted to earlier
observations and then frozen. This distinction lets us study correction
both when the answer mechanism is known and when errors come from trained
predictors. Within each dataset, all methods use the same prepared stream
and availability schedule. Appendix~\ref{app:benchmark-queries} gives the
evaluation protocol and consultation fees; we describe the inputs and
experts here. Full generator and model configurations are retained with
the experimental records.

\subsection{Shared preprocessing}
\label{app:state-construction}
\textbf{Information available to the learner.}
All routers receive the same public features, internal prediction,
availability, and fees. The features describe the current context and,
on real data, the internal predictor's past errors. They never contain
the current outcome or an unpurchased expert answer.
Melbourne and Delhi use only lagged measurements and calendar features.
Jena additionally uses weather covariates observed at the target timestamp,
before temperature is revealed.

\textbf{Feature scaling.}
We standardize continuous covariates using the dataset's preprocessing
prefix; calendar sine--cosine terms and the linear trend remain as
constructed. Melbourne uses its fitting block, which is also replayed
for offline warmup. Jena uses its full fitting-and-warmup prefix, with the
transform fixed before replay. No scored observation enters these fits.
For the real-data streams, the regression state is
\[
 x_t^{\rm ctx}=(p+3)^{-1/2}
 \bigl(1,\clip_{[-3,3]}(\widetilde x_t)^\top/3,
 v_{t-1},\bar v_{t-1}\bigr)^\top,
\]
where $\widetilde x_t\in\R^p$ contains the preprocessed public features.
The two error coordinates are bounded logarithmic transforms of the
internal predictor's previous squared error and its exponentially
weighted average. Both use only past outcomes. Each coordinate inside
the parentheses lies in $[-1,1]$, so $\|x_t^{\rm ctx}\|_2\leq1$.
In \ORUCB{}, we use $x_t^0=s_{t,j}=\ell_{t,j}=x_t^{\rm ctx}$;
expert identity enters through the registry block of $c_{t,j}$.
The baselines receive the same context vector without that expansion.

\textbf{Prediction units.}
On real data, we divide outcomes and predictions by the internal
predictor's fitting root mean squared error $s$, and fees by $s^2$.
We convert reported costs back to squared-temperature units; synthetic
data use $s=1$. The physical residual and disagreement bounds
$(R_{\rm phys},D_{\rm phys})$ are $(1.2,1.3)$, $(15,20)$, $(15,25)$,
and $(8,14)$ for synthetic, Melbourne, Jena, and Delhi, respectively.
The learner uses $R=R_{\rm phys}/s$, $D=D_{\rm phys}/s$, and $G=R+D$.
Internal and frozen-model predictions are clipped to $[-50,60]$ on real
data; corrected predictions receive no additional clipping. Bounds are
checked without truncating observations.

\subsection{Synthetic}
\label{app:construction-synthetic}
The synthetic stream separates four situations: free prediction is
adequate, an expert can be used directly, an expert needs rescaling, or
an expert needs nonlinear correction. These regimes recur every $240$
rounds, with $60$ rounds per regime.
The public context contains a constant and four sine--cosine pairs of
this phase, normalized by $\sqrt{5}$.

The internal prediction $h_t$ is a fixed periodic signal. In regime $r$,
the outcome is $Y_t=h_t+\mu_r(U_t)+\epsilon_t$, where $U_t$ is a symmetric
seven-point latent signal on $[-1,1]$ and $\epsilon_t$ is mean-zero noise
on $\{-0.2,0,0.2\}$. The regime functions are
\[
\begin{gathered}
\mu_0(u)=0,\qquad \mu_1(u)=u,\qquad \mu_2(u)=0.65u,\\
\mu_3(u)=0.12u+0.30\tanh(3u)+0.50(u^2-\E U_t^2).
\end{gathered}
\]
The last regime includes a hyperbolic-tangent term rather than using a
cubic polynomial as its generating function.

Twelve experts form four families with three identities each. Their
answers are noisy, biased linear measurements of $U_t$ added to $h_t$.
Three families specialize in the informative regimes; the fourth gives
a noisier signal across regimes. Shared and identity-specific noise make
expert answers dependent without making them identical.
One identity per family is available at a time, for $150$ rounds followed
by $300$ rounds of absence. The generated outcomes and all potential
answers are fixed across learner seeds.

\subsection{Melbourne}
\label{app:construction-melbourne}
We use daily minimum temperatures from Melbourne
\citep{daily_min_temp_melbourne}, with $365$ initial observations reserved
to form lags. Public features are temperature lags $1,2,7,14,30,365$,
weekday and annual sine--cosine pairs, the second annual harmonic, and a
linear trend. An internal ridge predictor is fitted before the online
replay and subsequently updated from revealed outcomes.

The twelve experts are simulated sensors, not trained forecasters.
They measure the current temperature through four response types:
identity, two linear rescalings, and a saturating hyperbolic-tangent
transform. Specifically, the transforms of the clipped standardized
residual $u_t=\clip_{[-5,5]}((Y_t-h_t)/s)$ are
$u_t$, $0.65u_t$, $\tanh(1.8u_t)/1.8$, and $0.8u_t$.
Sensor reliability varies with season and recent temperature change,
measured by the absolute difference between standardized lag-one and lag-seven
features relative to its fitting-prefix distribution. Lower reliability
attenuates the signal and increases noise. Each family contains three
identities with small gain and bias differences; noise has both shared
and identity-specific components.

This construction combines real temperature dynamics with controlled
measurement errors. Only the purchased sensor answer is revealed.
One identity per family is available each day; identities rotate after
$25$ days, giving each a $50$-day absence before its next active period.

\subsection{Jena}
\label{app:construction-jena}
We sample the Jena climate series \citep{jena_climate} every four hours
and predict temperature. Its $13$ public features are pressure, humidity,
wind speed, temperature lags $1,6,42,180$, and hour, weekday, and annual
sine--cosine pairs. Pressure, humidity, and wind speed are contemporaneous
measurements; all temperature inputs are lagged.

The twelve frozen experts belong to four families: random forests,
extra trees, gradient-boosted trees, and RBF kernel ridge. Each family
contains three models with different depths or kernel widths, giving
distinct predictors without changing the prediction task. These models
are fitted on the initial $2000$ observations using scikit-learn
\citep{pedregosa2011scikit}. They remain frozen while the common internal
ridge predictor updates after each outcome.

One identity per family is available at a time. The active identity
changes every $84$ rounds, corresponding to $14$ days of availability
followed by $28$ days of absence. The schedule is fixed independently of
the router's purchases.

\subsection{Delhi}
\label{app:construction-delhi}
We predict daily mean temperature in the Delhi climate series
\citep{daily_delhi_climate}. The $16$ public features are one-day-lagged
humidity, wind speed, and pressure; temperature lags $1,2,7,14,30,365$;
weekday and annual sine--cosine pairs; the second annual harmonic; and
a linear trend. Unlike Jena, this stream uses no contemporaneous weather
measurements.

Six frozen scikit-learn models \citep{pedregosa2011scikit} are fitted on
the initial $600$ transformed observations: two random forests, two
extra-trees models, and two gradient-boosted models. Different tree
depths distinguish the two experts within each family. The common
internal ridge predictor is initialized on the same prefix and updates
after each revealed outcome.

Four to six experts are available per round. Two are always available;
the others follow a fixed schedule of absences, with one arriving during
warmup. Identities persist across the stream, and the availability
schedule is shared by every method.

%% file: appendix/i_related_details.tex
\section{Extended related work}
\label{app:related-details}

\textbf{Learning to defer.}
L2D extends selective prediction by replacing rejection with a decision
to rely on an external expert. Selective prediction studies the
error--rejection tradeoff \citep{chow1970reject}, convex rejection losses
\citep{Bartlett_Wegkamp_2008}, and abstention with deep networks
\citep{Geifman_El-Yaniv_2017}. Predictor--rejector formulations learn
prediction and rejection scores
\citep{cortes,cortes2024theory,ni2019calibration,Mao_Mohri_Zhong_2023};
other formulations accommodate arbitrary task losses
\citep{cao2022generalizing} or score rejection as an additional action
\citep{theoretically}.
Deferral instead accounts for the receiving expert's errors
\citep{madras2018predict,mozannar2020consistent}, with multiple-expert
methods comparing input-dependent performance
\citep{Verma2022LearningTD,mao2024principled} or selecting a set of experts
\citep{montreuil2026why}.
This dependence on expert capabilities motivates work on fairness and
accuracy in routing \citep{Keswani}, complementary human--AI teams
\citep{Hemmer}, sample-efficient training with expert predictions
\citep{charusaie2022sample}, and fine-tuning for individual experts
\citep{raman2021improvinglearningtodeferalgorithmsfinetuning}.
Joint predictor--router training has been studied through exact
optimization for linear models \citep{Mozannar2023WhoSP}, softmax
parameterizations for probability estimation
\citep{Cao_Mozannar_Feng_Wei_An_2023}, and losses that mitigate classifier
underfitting \citep{liu2024mitigating}.
Two-stage methods instead fit the predictor before learning the router
\citep{mao2023twostage}, while \citet{montreuil2026beyond} separate
classifier and expert-utility heads during training.
The observations available for learning are another distinction:
budgeted deferral queries a subset of experts on each labeled example
\citep{desalvo2025budgeted}, and probabilistic deferral handles missing
expert annotations while controlling workload allocation
\citep{nguyen2025probabilistic}.
Other methods accommodate experts unseen during training through
meta-learning \citep{tailor2024population} or identity-free deferral
\citep{strong2026identityfree}. Online routing with varying experts and
selected-action feedback \citep{duy2026online} and deferral in
non-stationary time series \citep{montreuil2026slds} address changes
during deployment.
Beyond classification, L2D has been developed for regression
\citep{mao2024regression}, multi-task learning
\citep{montreuil2024twostagelearningtodefermultitasklearning}, and
adversarial robustness in both two-stage and one-stage systems
\citep{montreuil2025adversarialrobustnesstwostagelearningtodefer,montreuil2026adversarial}.
Applications include sequential medical decisions
\citep{joshi2021learning}, guided deferral with language models in
healthcare \citep{strong2024towards,strong2025trustworthy}, and query
allocation in extractive question answering \citep{montreuil2026optimal}.
Causal approaches study counterfactual second opinions \citep{Benz} and
the evaluation of deferring systems \citep{palomba2025a}.
Expert input can also inform a prediction rather than replace it:
\citet{kerrigan2021combining} combine human predictions with model
probabilities, and \citet{pugnana2025ask} learn when to request human
feedback for an enriched predictor.
Related acquisition methods route to a classifier with additional
features \citep{regol2025acquisition} or supply advice to a selected
expert \citep{montreuil2026advice}.
Our response instead transforms the purchased answer. We learn both
acquisition and responses online, fitting the latter only on purchased
answers. The router must therefore distinguish the loss of its current
response from the risk achievable with the best feasible response.

\medskip\noindent\textbf{Online routing and bandits.}
The choice of an expert is an action-selection problem with partial
feedback. Linear contextual bandits estimate conditional action values
and use confidence bounds to balance prediction and exploration
\citep{li2010contextual,abbasi2011improved}; Thompson sampling uses parameter
randomization \citep{agrawal2013thompson}, and NeuralUCB uses neural
gradient features for its confidence model \citep{zhou2020neuralucb}.
Regression-oracle methods provide another route from prediction to
contextual-bandit decisions \citep{foster2020beyond}.
Sleeping bandits restrict each round's choice to the currently available
actions \citep{kleinberg2010sleeping}. These ideas underlie
SleepingLinUCB, SharedLinUCB, LinTS, and NeuralUCB in our comparison.
We also include mechanisms for changing action values: discounted linear
regression in D-LinUCB \citep{russac2019weighted}, and detector-triggered
restarts in CUSUM-LinUCB and GLR-LinUCB, adapted from the non-contextual
methods of \citet{liu2018change} and \citet{besson2022efficient}.
The seven external baselines follow the same availability and feedback
protocol and deploy purchased answers unchanged; separate matched-response
comparisons retain our correction learner and vary the router.
Our confidence weighting follows \citet{he2022nearly}, who study corrupted
linear bandits. In our setting, response learning generates the discrepancy
from the linear target; bounding this discrepancy lets us calibrate
acquisition.

\medskip\noindent\textbf{Paid information and aggregation.}
Learning to buy information is related to, but distinct from, choosing an
expert whose answer is the final prediction. \citet{janisch2019costly}
sequentially acquire feature values before classification, learning a
policy from episodes sampled from labeled data. Paid-observation models
also control which advice is obtained
\citep{seldin2014prediction,reyzin2026costly}, with different prediction
and feedback structures.
Stacked generalization likewise learns how to use other predictors'
outputs \citep{wolpert1992stacked,hastie2009elements}.
In our setting, however, the outputs available for fitting are selected
by the router, and the resulting fitting errors affect future purchases.
A purchased answer and the revealed outcome suffice to evaluate both
response proposals and combine them by exponential weighting
\citep{cesabianchi2006prediction}.

\medskip\noindent\textbf{Polynomial responses.}
Polynomial features allow nonlinear prediction with coefficients learned
by linear methods. In online decision problems,
\citet{hu2020smooth} estimate contextual rewards by local polynomial
regression, while \citet{liu2021smooth} use local Taylor approximations
with misspecified linear-bandit algorithms. Smoothness-adaptive contextual
bandits address the choice of approximation scale
\citep{gur2022smoothness}, and \citet{maran2024smooth} use Legendre features
to represent smooth state-action functions in reinforcement learning.
Unlike these context- or action-based approximations, our dictionary
transforms a purchased answer. The regret proof uses bounded features
and linear coefficients, without smoothness or distributional
orthogonality assumptions.